\documentclass{article}

    \usepackage[final]{neurips_2025}

\usepackage{amssymb}
\usepackage{booktabs}
\usepackage{tikz}
\usepackage{array}
\usepackage{makecell}
\usepackage{setspace} 
\usepackage{tabularx}
\usepackage[utf8]{inputenc} % allow utf-8 input
\usepackage[T1]{fontenc}    % use 8-bit T1 fonts
\usepackage[hidelinks]{hyperref}
\usepackage{url}            % simple URL typesetting
\usepackage{booktabs}       % professional-quality tables
\usepackage{amsfonts}       % blackboard math symbols
\usepackage{nicefrac}       % compact symbols for 1/2, etc.
\usepackage{amsthm}
\usepackage{amsmath}
\usepackage{algorithm}
\usepackage{algpseudocode}
\newtheorem{definition}{Definition}
\newtheorem{theorem}{Theorem}

\newtheorem{lemma}{Lemma}
\newtheorem{proposition}{Proposition}
\newtheorem{assump}{Assumption}
\newtheorem{fact}{Fact}

\usepackage{microtype}      % microtypography
\usepackage{xcolor}         % colors
\newcommand*{\infn}[1]{\left\|{#1}\right\|_{\infty}}
\newcommand{\cA}{\mathcal{A}}

\newcommand{\cG}{\mathcal{G}}

\newcommand{\cL}{\mathcal{L}}
\newcommand{\cS}{\mathcal{S}}
\newcommand{\cF}{\mathcal{F}}
\newcommand{\cM}{\mathcal{M}}
\newcommand{\cN}{\mathcal{N}}
\newcommand{\cP}{\mathcal{P}}

\newcommand{\cX}{\mathcal{X}}

\newcommand{\cZ}{\mathcal{Z}}

\newcommand{\expec}{\mathbb{E}}

\newcommand{\real}{\mathbb{R}}

\newcolumntype{L}[1]{>{\raggedright\arraybackslash}p{#1}}

\DeclareMathOperator*{\argmin}{argmin}
\DeclareMathOperator*{\argmax}{argmax}

\title{
Finite-Time Bounds for \\
Average-Reward
Fitted Q-Iteration
}

\author{Jongmin Lee\\
Seoul National University\\
Department of Mathematical Sciences\\
\texttt{dlwhd2000@snu.ac.kr} \\
\And
Ernest K. Ryu \\
UCLA \\
Department of Mathematics \\
\texttt{eryu@math.ucla.edu} 
}

\begin{document}

\maketitle

\begin{abstract}
Although there is an extensive body of work characterizing the sample complexity of discounted-return offline RL with function approximations, prior work on the average-reward setting has received significantly less attention, and existing approaches rely on restrictive assumptions, such as ergodicity or linearity of the MDP. In this work, we establish the first sample complexity results for average-reward offline RL with function approximation for weakly communicating MDPs, a much milder assumption. To this end, we introduce Anchored Fitted Q-Iteration, which combines the standard Fitted Q-Iteration with an anchor mechanism. We show that the anchor, which can be interpreted as a form of weight decay, is crucial for enabling finite-time analysis in the average-reward setting. We also extend our finite-time analysis to the setup where the dataset is generated from a single-trajectory rather than IID transitions, again leveraging the anchor mechanism.
\end{abstract}

\section{Introduction}

The goal of offline Reinforcement Learning (RL) is to find a near-optimal policy using a precollected dataset without any direct interaction with the environment. Characterizing the sample complexity for finding an $\epsilon$-optimal policy using function approximation under assumptions that the offline data has sufficient coverage over the whole state-action space has been an active area of theoretical RL research. However, more prior work focuses on the discounted cumulative reward setup, and research on obtaining sample complexity in the average reward has been limited due to the absence of the discount factor and the complexity of the Bellman equation. Specifically, all prior works with function approximation rely on restrictive assumptions such as ergodicity or linearity of the MDP. 

Although theoretical RL research often focuses on the discounted return setup due to the theoretical convenience offered by the discount factor and the simpler Bellman equation, many practical scenarios are more naturally modeled as agents aim to maximize the average reward. In fact, many practical RL applications do not use discounting at all. These considerations make the sample complexity of average-reward RL relevant, despite the additional technical challenges this setting presents.

\paragraph{Contribution.}
In this work, we introduce the Anchored Fitted Q-Iteration and establish the sample complexity on average reward MDPs with general function approximation for weakly communicating MDPs for the first time. We consider the cases with IID data and with single-trajectory data. Then, we show that by using the relative normalization mechanism from the classical relative value iteration, we can further improve the sample complexity.

% , we consider the capacity of the function set as measured by a covering number and VC-dimension, respectively. 

\begin{table}[h]
    \centering
    \begin{tabular}{ccccc}
        \toprule
\textbf{ Prior works} & \makecell{\textbf{MDP class} } & \makecell{\textbf{dataset}} &
        \makecell{\textbf{Coverage}\\\textbf{coefficient}}  \\
        \midrule
        Ozdaglar et al.\ \cite{ozdaglar2024}   & \makecell{ergodic$^*$}  & IID samples & partial  \\
        \midrule
        Gabbianelli et al.\ \cite{gabbianelli2024offline} & \makecell{unichain$^*$ (+ linear)  } & IID samples  & partial \\
        \midrule
        Our work & weakly communicating$^*$ & IID samples & full \\
        \midrule
        Our work & \makecell{weakly communicating$^*$  } &  \makecell{$\beta$-mixing single-trajectory} & full \\
        \bottomrule
    \end{tabular}
    
    \vspace{0.1in}
    
    \caption{Comparison of analyses of offline average-reward MDPs. Our work, which assumes the MDP is weakly communicating, significantly relaxes the structural assumption on the MDP compared to prior work. (Clarification$^*$: Ergodic, unichain, and weakly communicating are respectively the standard MDP classes for which the results of \cite{ozdaglar2024}, \cite{gabbianelli2024offline}, and our work apply. However, the precise conditions are slightly more general in each case. See Section~\ref{ss:mdp-class} for detailed definitions.) 
    }
    \label{table:main}
\end{table}

% \begin{table}[h]
%     \centering
%  \begin{tabular}{ |c|c|c|c|c|c|c|c|c| } 
%  \hline
%  Prior works& \makecell{Structure of MDP}  & \makecell{dataset} &
%   \makecell{Coverage\\coefficient}&
%   \makecell{Function\\ Approximation} & \makecell{Sample \\complexity}\\
%   \hline   
%     \cite{ozdaglar2022revisiting}   & \makecell{uniform mixing \\ uniform Bellman condition}  & IID samples & single &general & $O$\\
%     \hline
%     \cite{gabbianelli2024offline} &\makecell{linear MDP \\ uniform Bellman condition}   & IID samples  & single &linear & $O$\\
%     % $[5]$ &&&  &   \\
%   \hline
%   Our work & single Bellman condition    &  IID samples & uniform &general&$O$\\
%   \hline
%   Our work &  \makecell{single policy mixing \\single Bellman condition}   & one trajectory & uniform &general&$O$\\
%   \hline
% \end{tabular}
% \caption{ Offline RL algorithms for average reward MDPs, uniform mixing $\subset$ Bellman uniform equation $\subset$ Bellman optimality equation XXXX value vased policy vased general function approximation}
% \end{table}

\newpage

\begin{figure}
    \vspace{-0.1in}
    \centering
\begin{tikzpicture}
    % Outer rectangle
    \draw[thick, rounded corners] (-2.9, 1.5) rectangle (2.9, -1.2);
    \node at (0, 1.05) {\fontsize{10.5}{5}\selectfont weakly communicating };
    
    % Middle rectangle
    \draw[thick, rounded corners] (-2.1, 0.6) rectangle (2.1, -1.1);
    \node at (0, 0.2) {\fontsize{10.5}{5}\selectfont unichain};
    
    % Inner rectangle
    \draw[thick, rounded corners] (-1.4, -0.2) rectangle (1.4, -1.0);
    \node at (0, -0.6) {\fontsize{10.5}{5}\selectfont ergodic};
\end{tikzpicture}
    \caption{The MDP classes satisfy the inclusion: ergodic $\subset$ unichain $\subset$ weakly communicating 
    % {\color{red} make boxes slightly less tall}
    }
    \label{fig:class_MDP}
\end{figure}
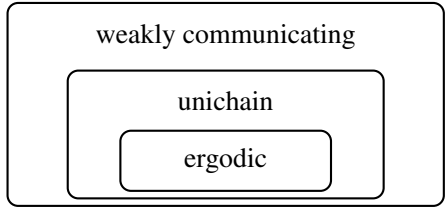

\subsection{Preliminaries and notations}
We briefly review the basic notions of average-reward Markov decision processes (MDPs) and reinforcement learning (RL) and refer the readers to standard references for further details \cite{10.5555/528623, bertsekas2015dynamic, sutton2018reinforcement}. 
\paragraph{Average-reward MDP.}
Let $\cM(\cX)$ be the space of probability distributions over $\cX$ and $\cF(\cX)$ be the space of bounded real-valued functions over $\cX$. Write $(\cS, \cA, P, r)$ to denote an infinite-horizon undiscounted MDP with finite state space $\cS$, finite action space $\cA$, transition matrix $P\colon \cS \times \cA \rightarrow \cM(\cS)$, bounded reward $r\colon  \cS \times \cA \rightarrow [-R,R]$. Denote $\pi\colon \cS \rightarrow \cM(\cA)$ for a policy, $g^{\pi}(s,a)=\liminf_{T\rightarrow \infty} \frac{1}{T}\expec_{\pi}\left[\sum^{T}_{t=1}  r(s_t, a_t) \,|\, s_0=s, a_0=a\right]$ for the average-reward of a policy $\pi$ given an initial state-action pair $(s,a)$,  where $\expec_{\pi}$ denotes the expectation over all trajectories $(s_0, a_0, s_1, a_1, \dots, s_{T}, a_{T})$ induced by $P$ and $\pi$. 

We say $\pi_\star$ is an optimal policy if $g^{\pi_\star}(s,a)=\max_{\pi}g^{\pi}(s,a)$ for all $s \in \cS$ and $a \in \cA$, and we say $g^{\pi_\star}$ is the optimal average reward.
 (The optimal policy and optimal average reward exists for finite state-action space \cite[Theorem 9.1.8]{10.5555/528623}.)
We say $\pi$ is an $\epsilon$-optimal policy if $\infn{g^{\pi_\star}-g^{\pi}} \le \epsilon$. 
% {\color{red} XXX let's not formally define $\epsilon$-optimal, since it is clunky without having specified the norm.}
Define $\cP^\pi$ as
\begin{align*}
%     \cP^{\pi}(s\rightarrow s')&=
% \mathrm{Prob}(s\rightarrow s'\,|\,
% a \sim \pi(\cdot\,|\,s), s'\sim P(\cdot\,|\,s,a))
% \\
\cP^{\pi}((s,a)\rightarrow (s',a'))&=
\mathrm{Prob}((s,a)\rightarrow (s',a')\,|\,
s'\sim P(\cdot\,|\,s,a),a' \sim \pi(\cdot\,|\,s')),
\end{align*}
% For notational conciseness, we write $T^{\pi}V=r^{\pi}+\gamma\cP^{\pi}V$ and $T^{\pi}Q=r+\gamma\cP^{\pi}Q$
the transition matrix induced by policy $\pi$. Then, $(\cP^\pi Q)(s,a) = \expec_{a' \sim \pi(\cdot \,|\, s'), s' \sim P(\cdot \,|\, s,a) }[Q(s',a')]$ for $Q \in \cF(\cS \times \cA)$. 
%For distribution $\rho$ on $\cS \times \cA$, denote $\rho\cP^\pi$ as distribution where $(s',a') \sim \rho\cP^\pi$ is defined by $(s,a) \sim \rho, s' \sim \cP(\cdot \,|\, s,a), a' \sim \pi(\cdot \,|\, s')$. 
% $\rho\cP^\pi (s,a) = \expec_{(s',a') \sim \rho }[\cP^\pi(s',a')]$ for distribution $\rho$ on $\cS \times \cA$. 
 Define the weighted $L_p$-norm of $Q \in \cF(\cS \times \cA)$ under  state-action distribution $\rho$ as $\|Q\|_{p, \rho} = [\mathbb{E}_{(s,a)\sim \rho}|Q(s,a)|^p]
^{1/p}$ for $p \ge 1$.

\paragraph{Coverage coefficient.} 
A coverage coefficient quantifies the shift between the distribution of the offline data and the distribution induced by policies \cite{munos2008finite, chen2019information}. Loosely speaking, the \emph{full coverage} assumption, as stated in Table~\ref{table:main}, assumes that the offline data sufficiently explores the whole state-action space regardless of policy \cite{antos2008learning,xie2021batch}, while \emph{partial coverage} only requires the offline data to sufficiently explore the state-action pairs that an optimal policy would visit \cite{zhan2022offline, jin2021pessimism}. These types of assumptions are fundamentally necessary for the complexity analysis of offline RL \cite{chen2019information}, and different works use different types of coverage coefficients (cf.\ \cite{uehara2021pessimistic,scherrer2014approximate}). The coverage coefficient we use is defined in Section \ref{sec::Apx-Anc-QI}.

\paragraph{Value Iteration.}
Given an undiscounted MDP $(\cS, \cA, P, r)$, the Bellman optimality operator $T$ is
\begin{align*}
% TV(s)&=\max_{a \in \cA} \left\{r(s,a)+\mathbb{E}_{s'\sim P(\cdot\,|\,s,a) }\left[V(s')\right]\right\}
TQ(s,a)&=r(s,a)+\mathbb{E}_{s'\sim P(\cdot\,|\,s,a)}\Big[\max_{a' \in \cA} Q(s',a')\Big]
\end{align*}
for all $s \in \cS$ and $a \in \cA$.
% For notational conciseness, we write $T^{\pi}V=r^{\pi}+\gamma\cP^{\pi}V$ and $T^{\pi}Q=r+\gamma\cP^{\pi}Q$ is transition matrix induced by policy $\pi$. 
We define the standard Value Iteration (VI) as 
\[ Q^{k}=TQ^{k-1} \qquad\text{ for } k=1,2,\dots,K,
\]
 where $Q^0$ is an initial point.
 % /We can express VI more concisely as $Q^k=r+\cP^{\pi_{k-1}} Q^{k-1}$ where $\pi_{k-1}(s)=\argmax_{a \in \cA} Q^{k-1}(\cdot \,|\,s)$.{\color{red} why is this needed? XXX}
 
 % After executing $K$ iterations, VI returns the near-optimal policy $\pi_K$ as a deterministic policy satisfying $\pi_K(s)=\argmax_{a \in \cA} Q^K(\cdot \,|\,s).$ 

\paragraph{MDP classes.}
  MDPs are classified according to the structure of the transition matrices. (For definitions on irreducible classes, recurrent classes, transient states, and aperiodicity of transition matrices, refer to \cite[Appendix A.2]{10.5555/528623}.) 
  % /{\color{red} XXX is $P^\pi$ a "markov chain"? Is it s transition matrix? XXX}
  An MDP is \emph{ergodic} if the transition matrices induced by every policy $\pi$ has a single recurrent class and is aperiodic. An MDP is \emph{unichain} if the transition matrices induced by every policy $\pi$ has a single recurrent class plus a possibly empty set of transient states. An MDP is \emph{weakly communicating} if there is a set of states where every state in the set is accessible from every other state in that set under some policy, plus a possibly empty set of states that are transient for all policies. Otherwise, in general, an MDP is \emph{multichain}. We note that classification of MDPs is crucial in the analyses of average-reward MDPs \cite{zhang2023sharper,zurek2024span, wei2021learning,lee2025multi}.

 % MDP is ergodic if MDP is unichain, MDP is weakly communicating if MDP is unichain, and MDP is multichain if MDP is weakly communicating. 
 
% In our work, we assumes single-policy Bellman single conditon for IID dataset and single mixing condition for single-trajectory dataset as we clarified in table. 

\subsection{Conditions of Prior works }\label{ss:mdp-class}
In Table~\ref{table:main}, we remarked that the precise conditions on the MDPs are slightly general than ergodic, unichain, and weakly communicating. In this section, we state the precise conditions.

\paragraph{Uniform mixing \cite{ozdaglar2024}.} The \emph{uniform mixing} condition assumes that there exist positive $t_{mix} \in \mathbb{N}$ 
(which does not depend on $\pi$ and $\rho$)
such that $\|\rho^\top (\cP^\pi)^t - \nu^\pi \|_1 \le 1/2$ for all $t\ge t_{max}$ for any policy $\pi$ and initial distribution $\rho$, where $\nu^\pi$ is the stationary distribution of $\cP^\pi$. 
(This condition requires that the stationary distribution $\nu^\pi$ is unique for all $\pi$.)
The prior work \cite{ozdaglar2024} uses this assumption for its analysis of average-reward MDPs. 
Ergodic MDPs satisfy the uniform mixing condition \cite{bradley2005basic,levin2017markov}, but  unichain MDPs do not  \cite[Example 8.2.1]{10.5555/528623}.

\paragraph{All-policy Bellman equation and linear MDPs \cite{gabbianelli2024offline}.}
The \emph{all-policy Bellman equation} states that for any policy $\pi$, the average reward $g^\pi$ does not depend on $(s,a)$ and there exist a $Q^\pi\colon \cS \times \cA \rightarrow \real$ such that
       \begin{align*}
r(s,a) +\expec_{a' \sim \pi(\cdot \,|\, s'), s' \sim P(\cdot \,|\, s,a)}\left[ Q^\pi(s',a')\right] =Q^\pi(s,a)+g^\pi
\end{align*}
for all $s\in \cS$ and $a\in \cA$. The prior work \cite{gabbianelli2024offline} uses this assumption for its analysis of average-reward MDP. Unichain MDPs satisfy the all-policy Bellman equation while weakly communicating MDPs, in general, do not \cite[Section~8.4] {10.5555/528623}. Also, the uniform mixing condition implies the all-policy Bellman equation \cite[Lemma 6]{wei2021learning}.

An MDP is \emph{linear} if there exist $\phi: \cS \times \cA \rightarrow \real^d$,  $\psi: \cS \rightarrow \real^d$, and $ w \in \real^d$ such that 
\[r(s,a) = \langle \phi(s,a), w \rangle, \quad P(s' \,|\, s,a) = \langle \phi(s,a), \psi(s') \rangle.\]
The linear MDP assumption is often used for theoretical analyses \cite{jin2020provably, gabbianelli2024offline}, but it requires knowledge of the mapping $\phi$ and $\psi$ and often fails to hold in practice \cite{ ghosh2017misspecified,vial2022improved}.

% Linear MDP assumes a transition kernel and reward function both linear in specific state-action feature representation, and \cite{gabbianelli2024offline} further assumes this.
% , and we note that this distinction is one of important factors in analyses of average-reward MDPs \cite{zhang2023sharper,zurek2024span, wei2021learning,lee2025multi}. 

\paragraph{Bellman optimality equation (our work).}
 \begin{assump}[Bellman optimality equation]\label{assum_bell_opt}
 The optimal average reward $g^{\pi_\star}$ does not depend on $(s,a)$ and there exist a  $Q^{\pi_\star}\colon \cS \times \cA \rightarrow \real$ such that
    \begin{align*}
r(s,a) +\expec_{s' \sim P(\cdot \,|\, s,a)}\left[\max_{a'} Q^{\pi_\star}(s',a')\right] =Q^{\pi_\star}(s,a)+g^{\pi_\star},\qquad
\end{align*}
for all $s\in \cS$ and $a\in \cA$.    
 \end{assump}
A policy $\pi_\star$ satisfying the Bellman optimality equation is an optimal policy \cite[Section 2]{wei2021learning}.
% Our work assumes the Bellman optimality equation.
The all-policy Bellman equation implies the Bellman optimality equation, and the weakly communicating condition of MDPs also implies the Bellman optimality equation \cite[Theorem~8.3.2, ~8.4.1]{10.5555/528623}. 

% /We say $(g^{\pi_\star}, Q^{\pi_\star})$ is a solution of the Bellman optimality equation if $(g^{\pi_\star},h^\star)$ satisfies equation. It is known that if the MDP is weakly communicating, optimal average reward $g^{\pi_\star}=\mathbf{1}g^{\pi_\star}$ and $\pi^\star$ attaining maximum in Bellman optimality equation is optimal average reward \cite[Theorem~8.3.2, ~8.4.1]{10.5555/528623}.

\section{Anchored Fitted Q-Iteration}
Consider the offline RL setup with precollected dataset $D = \{s_i, a_i,r_i, s'_i\}^n_{i=1}$, where $r_i=r(s_i, a_i)$ and $s_i' \sim P(\cdot \,|\, s_i, a_i)$.
Let $\cF$ be a nonempty function space to approximate $Q$-value. 
% $\cF = \{f : \cS \times \cA \rightarrow [-f_{max}, f_{max}] \,|\, f \in B(S\times A)\}$
We now introduce our novel algorithm, \emph{Anchored Fitted Q-Iteration (Anc-F-QI)}.
\begin{algorithm}
\caption{Anchored Fitted Q-Iteration $(D, K, \{\cF_i\}^K_{i=1} \{\lambda_{i}\}^K_{i=1})$}\label{Anc-F-QI}
\begin{algorithmic}
\State \textbf{Input}: $D = \{s_i, a_i,r_i, s'_i\}^n_{i=1}$, $f_0=0 $, $K\ge 1$, $\{\lambda_i\}^K_{i=1} \subset (0,1)$
      \For{$k=0,1,\dots, K-1$}
      \State $\hat{T}f_{k} =\argmin_{f \in \cF_{k+1}} \sum^{n}_{i=1} \big(f(s_i, a_i)-r_i-\max_{a\in \cA} f_k(s'_i, a)\big)^2$ 
      \State $f_{k+1}=  (1- \lambda_{k+1})f_0+ \lambda_{k+1}\hat{T}f_{k}$ 
      \qquad 
      {\color{gray}$\rhd$\;With $f_0=0$, this is weight decay}
      \EndFor
      \State $\pi(a \,|\, s) = \argmax_{a\in \cA} f_K(s,a)$
      \State \textbf{Output} $\pi, f_K$
\end{algorithmic}
\end{algorithm}
% \begin{algorithm}
% \caption{Anchored Fitted Q-Iteration $(D, f_0, K, \{\lambda_{i}\}^K_{i=1})$}\label{Anc-F-QI}
% \begin{algorithmic}
% \State \textbf{Input}: $D = \{s_i, a_i,r_i, s'_i\}^n_{i=1}$, $f_0 \in \cF$, $K\ge 1$, $\{\lambda_i\}^K_{i=1} \subset (0,1)$
%       \For{$k=0,1,\dots, K-1$}
%       \State $\hat{T}f_{k} =\argmin_{f \in \cF} \sum^{n}_{i=1} \big(f(s_i, a_i)-r(s_i,a_i)-\max_a f_k(s'_i, a)\big)^2$ 
%       \State $f_{k+1}=  (1- \lambda_{k+1})f_0+ \lambda_{k+1}\hat{T}f_{k}$ 
%       \qquad 
%       {\color{gray}$\rhd$\;This is weight decay if $f_0=0$}
%       \EndFor
%       \State $\pi(a \,|\, s) = \argmax_{a\in \cA} f_K(s,a)$
%       \State \textbf{Output} $\pi, f_K$
% \end{algorithmic}
% \end{algorithm}

% In Section \ref{sec::Apx-Anc-QI}, we conducted an $L_p$ bound analysis that is used
In Section~\ref{sec::sam_comp}, we present our sample complexity results for Anc-F-QI. Roughly speaking, Theorem~\ref{Sam_com_iid} establishes $\tilde{\mathcal{O}}(1/\epsilon^6)$ sample complexity with IID data and Theorem~\ref{Sam_com_traj} establishes $\tilde{\mathcal{O}}(1/\epsilon^{12})$ sample complexity with $\beta$-mixing single-trajectory data.
In Section \ref{sec::rel_anc}, to further improve the sample complexity with the \emph{Relative Anchored} Fitted-Q Iteration, establishing $\tilde{\mathcal{O}}(1/\epsilon^4)$ and $\tilde{\mathcal{O}}(1/\epsilon^8)$ sample complexities for IID and single-trajectory data, respectively. 

\subsection{The anchor mechanism and weight decay}
Our method Anc-F-QI stated as Algorithm~\ref{Anc-F-QI} consists of two main components. 
The first component is the first line of the for-loop, the classical Fitted Q-Iteration \cite{ernst2005tree, munos2008finite} step without discount factor. Its goal is to find the function $\hat{T}f_k \approx Tf_k$, where $T$ is the Bellman operator.
% Later in Lemmas~\ref{appx_bel_iid} and \ref{appx_bel_traj}, we show $\hat{T}f_k \approx Tf_k$ when $n$ is large.
However, unlike Fitted Q-Iteration in the discounted cumulative reward case, $\hat{T}f_k \approx Tf_k$ is not enough to establish a finite sample complexity in the average-reward setup. In the tabular case where the Fitted Q-Iteration reduces to Value Iteration (VI), it is known that VI might not converge. Specifically, there exists an average-reward MDP such that the policy error of VI does not converge to zero \citep[Example~4]{della2012illustrated}. Even if an aperiodicity condition is assumed, VI guarantees only asymptotic convergence without any known explicit convergence rate in the average-reward setup \cite [Theorem~9.4.5]{10.5555/528623}.  
% Due to unspecified convergence rate, asymptotic result of VI does not allow to conduct finite time-bound analysis. 

%{\color{red} (XXX??? what is unspecified convergence rate? Is it proven to be impossible to get a convergence rate in those setups? XXX)}
% Morevover, result of most prior works provide convergence rate assumes constraints on structure of MDP such as ergodic condition.and that's why we use anchoring in next line

Recently, Anchored Value Iteration (Anc-VI) was proposed to obtain finite-time bounds of policy error for average-reward MDPs \cite{bravostochastic,lee2025multi,lee2025near}.
Particularly, the \emph{Anchored Q-Value Iteration} is
  \[Q^{k} = (1-\lambda_k) Q^0+\lambda_k TQ^{k-1} \qquad\text{ for } k=1,2,\dots. \tag{Anc-QI}\]
where $\lambda_k$ parameter is to be chosen. Compared to the standard VI, Anc-QI obtains the next iterate as a convex combination between the output of $T$ and the \emph{starting point} $Q^0$. We call the $(1-\lambda_k)Q_0$ term the \emph{anchor term} since it serves to pull the iterates back toward the starting point $Q_0$. With this Anc-VI, \cite{lee2025multi} establish non-asymptotic convergence in the average-reward setup. Specifically,  Anc-VI exhibits the $O(1/k)$-rate in terms of policy error \cite{lee2025multi}[Theorem 2 and Corollary 2] without any restrictions on the MDP. 
% When $Q^0=0$, the anchor mechanism can be interpreted as an instance of weight decay.

% XXX  for sample complexity in Theorem \ref{Sam_com_iid} and \ref{Sam_com_traj}, we select $f_0=0$. If $\cF$ is parameterized neural network with last linear layer, interestinlgy, this anchoring part could be interpreted as weight decay \cite{farebrother2018generalization} since anchoring is just equal to multiply $\lambda_k$ to weight of last layer.XXX

This anchoring mechanism, classically also known as the Halpern iteration \cite{halpern1967fixed}, has been widely studied in minimax optimization and fixed-point problems \cite{sabach2017first, Lieder2021halpern, park2022exact, contreras2022optimal, yoon2021accelerated}. In the context of reinforcement learning, \cite{lee2024accelerating, lee2025multi} applied the anchoring mechanism to VIs for cumulative-return and average-reward MDPs under the tabular setting, and \cite{bravostochastic, lee2025near} applied the anchoring mechanism to Q-Value Iteration for cumulative-return and average reward MDPs under the generative model setting. 

In this work, we combine Fitted Q-Iteration with anchoring, as shown in the second line of the for-loop of Algorithm~\ref{Anc-F-QI}, and establish finite-time bounds on the sample complexity.

% We show that our Anchored Fitted Q-Iteration provides finite-time bound with proper assumption in following sections.

% sample complexities.
% with offline data in average-reward setup and obtained valid sample complexities. For that, we first study approximate anchored iteration in Section \ref{sec::Apx-Anc-QI} and \ref{sec::L_p_bound}
% to obtain $L_p$ bound and extend to offline setup in Section \ref{sec::sam_comp}. For this framework, the key In our analyses is Lemma \ref{lem::1} which shows the benefit of anchoring unlike standard value iteration. As \cite{munos2005error} did, this Lemma leads us to sample complexity on offline setup. We will clarify this point in next section. 

% In this work, we firstly applied this anchoring mechanism to fiited Q-Iteration with offline data in average-reward setup and obtained valid sample complexities. For that, we first study approximate anchored iteration in Section \ref{sec::Apx-Anc-QI} and \ref{sec::L_p_bound}
% to obtain $L_p$ bound and extend to offline setup in Section \ref{sec::sam_comp}. For this framework, the key In our analyses is Lemma \ref{lem::1} which shows the benefit of anchoring unlike standard value iteration. As \cite{munos2005error} did, this Lemma leads us to sample complexity on offline setup. We will clarify this point in next section. 

\subsection{Assumptions on the function space $\mathcal{F}$}

% In this function approixmation setup, instead of directly calcualte $Q$-function, we approximate it using our function ($f^k \approx Q^k$). 

% For the well-definedness of Anc-F-QI, we assume the following conditions on the function space $\cF$.  

% In this section, we define our assumptions on the function space $\mathcal{F}$.

\begin{assump}[existence of argmin]\label{assum_ex_arg}
\mbox{In Anc-F-QI, the argmin defining $\hat{T}f_k$ 
% $\argmin_{f \in \cF} \sum^{n}_{i=1} \frac{1}{n}(f(s_i, a_i)-r(s_i,a_i)-\max_a f_k(s'_i, a))^2$
exist for $k=0,\dots,K-1$. }
\end{assump}
This assumption is needed for the regression step of Algorithm \ref{Anc-F-QI} to be well defined.

% and we note that this assumption could be substituted to compactness of function space \cite{cucker2002mathematical}.  

% \begin{assump}[convex $\mathcal{F}$]\label{cvx_cpt_fc} $\mathcal{F}$ is convex, i.e., if $f,g \in \cF$, then
% $\eta f+(1-\eta)g \in \cF$ for all $\eta\in (0,1)$.
% \end{assump}
% This assumption implies that the anchor step of Algorithm \ref{Anc-F-QI} to be well defined. However, this assumption can be relaxed as follows. 

% This assumption is for well definedness of anchoring part to guarantee $\lambda_{k} f_0+ (1- \lambda_{k})\hat{T}f_{k} \in \cF$. Although we could conduct analyses defining error parameter, we assume this for simplicity of analyses. This assumption could be relaxed as follows.  

% \renewcommand{\theassump}{3'}
% \begin{assump}[star-shaped $\mathcal{F}$]\label{assum_star_fun}
% Given $0 \in \cF$, if $f \in \cF, (1-\eta)f_0+\eta f \in \cF$ for all $\eta\in (0,1)$.This star-shaped assumption is sufficient to ensure that the anchor step of Algorithm \ref{Anc-F-QI} is well defined.
% \end{assump}
\begin{assump}[star-shaped function space]\label{assum_star_fun}
If $f \in \cF, \eta f \in \cF$ for all $\eta\in [0,1]$.
\end{assump}
\renewcommand{\theassump}{\arabic{assump}}
This assumption implies that the anchor step of Anc-F-QI to be well defined. Star-shaped function space is a classical notion that relaxes convexity \cite{gardner1995geometric,hansen2020starshaped,leong2024star}, and if $\cF$ corresponds to a parametrized neural network with a linear layer as the output layer, $\mathcal{F}$ is star-shaped.

% In Anc-F-QI, anchoring part performs convex combination with only initial function $f_0$ Thus, previous convex function space assumption reduced to more mild condition called  
% Now, we introduce assumption related to error caused by function space. 

\begin{definition}[Inherent Bellman error]\label{bellman_err}
Define $\epsilon_B(\cF,\cF') = \max_{f \in \cF} \min_{f' \in \cF'} \|f'-Tf\| $ as the inherent Bellman error with respect to the norm $\|\cdot\|$.
\end{definition}
The inherent Bellman error $\epsilon_B$ quantifies the error due to the function spaces $\mathcal{F}, \mathcal{F'}$ in approximating the output of the Bellman operator \cite{munos2008finite,antos2007fitted,chen2019information}. Note that if the function spaces $\mathcal{F}, \mathcal{F'}$ are bounded (in the $\|\cdot\|_\infty$-norm), then $\epsilon_B $ is also bounded. 

\begin{assump}[Bellman completeness]\label{bellman_com}
$\epsilon_B(\mathcal{F}, \mathcal{F}')=0$, where $\epsilon_B$ the is inherent Bellman error. 
\end{assump}
Bellman completeness states that if $f \in \cF$, then $Tf \in \cF'$. I.e., $\cF,\cF'$ are closed under the Bellman operator. 
Although the Bellman completeness assumption is seemingly strong, it is often considered in sample complexity analyses in the offline RL literature \cite{chen2019information,fan2020theoretical}. In fact, the Bellman completeness condition is fundamental in the sense that the prior work \cite{foster2021offline} showed that a polynomial sample complexity cannot be established without Bellman completeness assumption.

% This assumption usually considered to be strong assumption, and several works tried to relax this condition to realizabilty condition ($Q^{\pi_\star} \in \cF$) \cite{munos2008finite, chen2019information}. 

% by showing that, roughly speaking, even with small coverage distribution mismatch, exponential lower bounds quickly arise in the absence of Bellman completeness. 

% Considering these assumptions, we will conduct sample complexity analyses of Anc-F-QI. But before going through offline setup, we will establish $L_p$ bound of Approximate Anchored Q-Value Iteration with coverage coefficient. 
% From the the next section, for the sample complexity in offline setup,
% \begin{align*}
% \hat{T}f_{k} &=\argmin_{f \in \cF} \sum^{n_k}_{i=1} (f(s_i, a_i)-r(s_i,a_i)-\max_a f_k(s'_i, a))^2\\
%     % f_{k+1}&= \argmin_{f \in \cF} \|f-(\lambda_{k} f_0+ (1- \lambda_{k})(\hat{T}f_{k} - g(f^{k})\mathbf{1}))\|^2
%         f_{k+1}&= \lambda_{k} f_0+ (1- \lambda_{k})\hat{T}f_{k} 
% \end{align*}

\section{Approximate Anchored Q-Value Iteration}\label{sec::Apx-Anc-QI}

In this section, we conduct an $L_p$ bound analysis that will later be used to establish the main sample complexity results of  Sections~\ref{sec::sam_comp} and \ref{sec::rel_anc}.
Define the \emph{Approximate Anchored Q-Value Iteration} as
\begin{align}
Q^{k} = (1-\lambda_k) Q^0+\lambda_k (TQ^{k-1}+\epsilon_k)
\tag{Apx-Anc-QI}\label{eq:Apx-Anc-QI}
\end{align}
for $k=1,2,\dots,K$, where $T$ is the Bellman operator, $Q^0\in \mathbb{R}^n$ is a starting point, and $\epsilon_k$ represents the evaluation error of $TQ^{k-1}$. We choose $\lambda_k=\frac{k}{k+2}$ for $k=1,\dots,K$, motivated by \cite{sabach2017first, contreras2022optimal}.

We now establish a convergence analysis of \ref{eq:Apx-Anc-QI} based on $L_p$ bounds of $\epsilon_k$.
Similar to the prior work \cite{munos2005error,munos2007performance,fan2020theoretical}, we assume the following coverage coefficient for our analysis.
% We defer the proofs to Appendix.

\begin{assump}[uniform stochastic transition]\label{converage_tran} 
For a given distribution $\mu$ on $\cS \times \cA$, 
% and the Radon-Nikodym derivative of $\cP$ with respecto $\mu$ is bounded uniformly
% There exists a constant $C_{\mu}$ such that, for all $(s,a) \in \cS \times \cA$ and policy $\pi$,
\[C_\mu
\stackrel{\mathrm{def}}{=}
\sup_{s,a, \pi} \infn{\frac{\cP^\pi(\cdot \,|\,s,a)}{\mu(\cdot)}}<\infty .\]
\end{assump}
% Then, $\sup_{s,a, \pi} \frac{\cP_*^\pi(\cdot \,|\,s,a)}{\mu(\cdot)} \le C_{\mu}$ by definition. Since $\int Q(x,y)\cP^\pi(dx,dy\,|\, s,a)$
\begin{assump}[uniform future state distribution]\label{converage_fut} 
For given distributions $\mu$ and $\rho$ on $\cS \times \cA$, 
% There exists constant $C_{\mu,\rho}$ such that for an arbitrary sequence of  policies $\{\pi_i\}^k_{i=0}$, and
\[ C_{\mu,\rho} \stackrel{\mathrm{def}}{=} \sup_{\pi_1,\pi_2, \dots \pi_k} \infn{\frac{ \rho^{\top} \cP^{\pi_1}\cP^{\pi_2}\cdots \cP^{\pi_k}(\cdot)}{\mu(\cdot)}} < \infty,\]
where $\pi_1,\pi_2, \dots \pi_k$ represents an arbitrary sequence of policies.
\end{assump}
% \begin{assump}[Uniform future state distribution]\label{converage_fut} 
% Given $\rho, \mu$ and for an arbitrary sequence of  policies $\{\pi_i\}^k_{i=0}$, assume that the $\rho^{\pi_0}\cP^{\pi_1}\cP^{\pi_2}\cdots \cP^{\pi_k}$ is absolutely continuous with respect to $\mu$ where $\rho^{\pi_0}(s, \pi_0(s)) = \rho(s)$ and
% \[C_{\mu,\rho} = \sup_{s,a, \pi_0,\pi_1, \dots \pi_k} \infn{\frac{d \rho^{\pi_0} \cP^{\pi_1}\cP^{\pi_2}\cdots \cP^{\pi_k}(\cdot\,|\,s,a)}{d\mu(\cdot)}} < \infty \]
% \end{assump}
The coverage coefficients measure the mismatch between the distribution of offline data and the distribution induced by the transition matrices and initial distributions. We note that Assumption \ref{converage_tran} implies Assumption \ref{converage_fut} with $C_{\mu,\rho} \le C_{\mu}$ \cite[Section 5]{munos2008finite}. 
% Since these coverage coefficient require bound for arbitrary policy, it is classifies as all-policy condition.
% With these coverage coefficients, we now preset  $L_p$ error bound of \ref{eq:Apx-Anc-QI}. 
% With this coverage assumption, we can obtain following $L_p$ bound from previous Lemma. 

\begin{proposition}\label{prop::1}
Let $p\in [1,\infty]$, and let $\mu$ and $\rho $ be distributions on $\cS \times \cA$.
  Under Assumption \ref{assum_bell_opt} and \ref{converage_tran} (Bellman optimality equation, uniform stochastic transition), the policy error of \ref{eq:Apx-Anc-QI} with $\lambda_k = \frac{k}{k+2}$ satisfies  
    \begin{align*}
       \| g^{\pi_\star}-g^{\pi_K}\|_{\infty}
       \le C^{1/p}_\mu \frac{8}{K+2}\|Q^{\pi_\star}-Q^0\|_{ p, \mu}+C^{1/p}_\mu \frac{2K}{3} \max_{1\le k \le K} \|\epsilon_k\|_{ p,\mu}.
    \end{align*}
Similarly, under Assumption \ref{assum_bell_opt} and \ref{converage_fut} (Bellman optimality equation, uniform future state distribution), the policy error of \ref{eq:Apx-Anc-QI} with $\lambda_k = \frac{k}{k+2}$ satisfies
    \begin{align*}
       \| g^{\pi_\star}-g^{\pi_K}\|_{p,\rho}
       \le C^{1/p}_{\mu,\rho}\frac{8}{K+2}\|Q^{\pi_\star}-Q^0\|_{ p, \mu}+C_{\mu,\rho}^{1/p}\frac{2K}{3}\max_{1\le k \le K} \|\epsilon_k\|_{ p,\mu}.
    \end{align*}
% Let $\lambda_i = \frac{i}{i+2}$. Under assumption \ref{converage_tran}, \ref{eq:Apx-Anc-QI} exhibits rate
    % \begin{align*}
    %    \| g^{\pi_\star}-g^{\pi_k}\|_{\infty}
    %    \le C^{1/p}_\mu\frac{4}{k+2}\|Q^{\pi_\star}-Q^0\|_{ p, \mu}+C^{1/p}_\mu\sum^{k}_{l=1} \frac{2(l+1)l}{(k+2)(k+1)}\|\epsilon_l\|_{ p,\mu}.
    % \end{align*}
% and under assumption \ref{converage_fut}, \ref{eq:Apx-Anc-QI} exhibits rate 
%     \begin{align*}
%        \| g^{\pi_\star}-g^{\pi_k}\|_{p,\nu}
%        \le C^{1/p}_{\mu,\nu}\frac{4}{k+2}\|Q^{\pi_\star}-Q^0\|_{ p, \mu}+C_{\mu,\nu}^{1/p}\sum^{k}_{l=1} \frac{2(l+1)l}{(k+2)(k+1)}\|\epsilon_i\|_{ p,\mu}
%     \end{align*}
\end{proposition}
To clarify, the $g$-, $Q$-, and $\epsilon$-terms in Proposition~\ref{prop::1} are functions of $(s,a)$ and the norms $\|\cdot\|_{p,\rho}$ and $\|\cdot\|_{p,\mu}$ are taking expectations with respect to the distributions $\rho$ and $\mu$.

% Then, due to previous lemma, we obtain following $L_p$ bound of approximate anchored Q iteration. 
% Note that each coverage coefficient provide $\infn{\cdot}$ and $\|\cdot\|_{p,\mu}$ bound of policy error, respectively. Compared to Proposition \ref{prop::1}, Theorem \ref{prop::1} provides upper bound of error with $\|\cdot\|_{ p,\mu}$-norm. 

The bounds of Proposition~\ref{prop::1} serve as the technical crux of our sample complexity results later presented in Theorems~\ref{Sam_com_iid}, \ref{Sam_com_traj}, \ref{Sam_com_iid_rel}, and \ref{Sam_com_traj_rel}.
In the bound, the first term decreases with order $\mathcal{O}(1/K)$ but the second error term increases with order $\Theta(K)$. Therefore, our subsequent arguments will ensure $\|\epsilon_k\|_{ p,\mu}=\mathcal{O}(1/K^2)$ by using sufficient offline samples.

\section{Sample complexity of Anchored Fitted Q-Iteration}\label{sec::sam_comp}

We now present sample complexity analyses of Anc-F-QI with IID and single-trajectory data.

\subsection{Range of function space $\mathcal{F}$}
\label{ss:f-range}
Before analyzing the complexity of those, we explain our issue on range of function space in average-reward setup and our choice of function space. 

When considering Fitted Q-Iteration in the discounted reward setup, the functions are often assumed to be bounded by $\infn{Q_\gamma^\star}$ \cite{munos2008finite, chen2019information},
% $\cF = \{f : \cS \times \cA \rightarrow [- Q_\gamma^\star,  Q_\gamma^\star] \,|\, f \in B(S\times A)\}$ 
where $Q_\gamma^\star$ is optimal state-action function with discount factor $\gamma$, since $\infn{f} \le \infn{Q_\gamma^\star}$ implies $\infn{Tf} \le \infn{Q_\gamma^\star}$. In the average-reward setup (without discounting), this property does not hold, and the Fitted Q-Iteration is expected to produce an unbounded sequence of functions. 
To address this issue, we allow the range of the function space to increase with each iteration.
\begin{assump}[increasing function range]\label{assum_inc_fun}
   Let $\cF_0=\{0\}$ and $\cF_k  \subset \{f : \cS \times \cA \rightarrow [- kR , k R ] \,:\, f \in B(S\times A)\}$ and $f_k \in \cF_k$ in Anc-F-QI for all $k$ . 
\end{assump}
Roughly speaking, 
\[
\|f_k\|\sim\infn{Tf_{k-1}}= \infn{r+ P \max_{a\in \cA} f_{k-1}}\lesssim R+\infn{f_{k-1}}\lesssim kR+\infn{f_0},
\]
so we increase the function bound as $  kR $.

% Thus, if $f_k \in \cF_k$, then $Tf_k \in \cF_{k+1}$, and it's proper choose the $\hat{T}f_k \approx Tf_k$ in $\cF_{k+1}$ in Anc-F-QI. For this sake of argument, we choose this increasing function range increasing range and then, our range of function space would depend on total number of iteration of Anc-F-VI. We will clarify that number of iteratiob will depend on $\epsilon$ error,  solution of Belman optimality equation, and coverage coefficient in next subsection.

\subsection{IID dataset}
% We first present analyze sample complexity with IID offline dataset. 
In this subsection, we study sample complexity with IID dataset.
\begin{assump}[IID dataset]\label{iid_data}
There is a distribution $\mu$ such that the dataset is $D = \{s_i, a_i,r_i, s'_i\}^n_{i=1}$ generated IID with $(s_i, a_i) \sim \mu$ and $ s_i' \sim P(\cdot \,|\, s_i, a_i)$ for $i=1,\dots,n$.
\end{assump}
Since we consider possibly infinite function space, as measurement of the capacity of function space, we use covering number  \cite{cucker2002mathematical, wainwright2019high}. 
\begin{definition}
    An $\epsilon$-cover of set $S$ with respect to metric $d$ is a set $\{\theta_i\}^N_{i=1} \subset S $ such that for all $\theta \in S$, there is an $i\in \{1,\dots,N\}$ such that $d(\theta, \theta^i) \le \epsilon$. The covering number $\cN(\epsilon; S, d)$ is the cardinality of the smallest $\epsilon$-cover.
    By convention, we define $\mathcal{N}(+\infty;S,d)=1$.
\end{definition}
We now present lemma which bounds approximation error of Anc-F-QI for IID dataset.
\begin{lemma}\label{appx_bel_iid}
Assume Assumptions \ref{assum_bell_opt}, \ref{assum_ex_arg}, \ref{assum_star_fun}, \ref{assum_inc_fun}, and \ref{iid_data} (Bellman optimality equation, existence of argmin, star-shaped function space, increasing function range, IID dataset). Let $\mu$ be the distribution generating the dataset. Let $\epsilon>0$ and $\delta>0$. 
% Assumption \ref{cvx_cpt_fc},\ref{bellman_err}, \ref{iid_data},     
With probability $1-\delta$, $\{f_k, \hat{T}f_k\}^{K-1}_{k=0}$ of Anc-F-QI with $\lambda_k=\frac{k}{k+2}$ satisfies 
\[\|Tf_k-\hat{T}f_k\|^2_{\mu,2} \le \frac{60(k+2)^2R^2  \ln(2KN_{k,\epsilon}N_{k+1,\epsilon}/\delta)}{n}+3\epsilon+13\epsilon_B(\cF_k, \cF_{k+1}),\]
where 
\[
N_{k,\epsilon}=\cN(\tfrac{\epsilon}{108(2k+1)R}; \cF_k, \infn{\cdot}) ,\quad\text{for }k=0,1,\dots,K-1.
\]
% \[\|Tf_k-\hat{T}f_k\|^2_{\mu,2} \le \frac{60(k+2)^2R^2  \ln(2KN_{k,\epsilon}N_{k+1,\epsilon}/\delta)}{n}+3\epsilon+13\epsilon_B,\]
% where 
% \[
% N_{k,\epsilon}=\cN(\tfrac{\epsilon}{108(2k+1)R)}; \cF_k, \infn{\cdot}).
% \]
\end{lemma}
We defer the proofs to Appendix~\ref{app:sec_4}, but we quickly note that the proof is based on Bernstein inequality and is motivated by \cite{cucker2002mathematical,chen2019information}. 

Lemma~\ref{appx_bel_iid} tells that the square of approximation error of the Bellman operator decreases sublinearly with respect to number of sample. Combining Theorem \ref{prop::1} and Lemma \ref{appx_bel_iid}, we obtain following sample complexity result of Anc-F-QI with IID dataset. 

\begin{theorem}\label{Sam_com_iid}
   Assume Assumptions \ref{assum_bell_opt}, \ref{assum_ex_arg}, \ref{assum_star_fun}, \ref{converage_tran}, \ref{assum_inc_fun}, and \ref{iid_data} (Bellman optimality equation, existence of argmin, star-shaped function space, uniform stochastic transition, increasing function range, IID dataset). 
   Let $\mu$ be the distribution generating the dataset.
   Let $\epsilon>0$ and $\delta>0$.
   With probability $1-\delta$, the policy error of Anc-F-QI with $\lambda_k=\frac{k}{k+2}$ and  $ K=\lceil 18C^{1/2}_\mu\|Q^{\pi_\star}\|_{2,\mu}/\epsilon\rceil$ satisfies $ \| g^{\pi_\star}-g^{\pi_K}\|_{\infty}
       \le \epsilon+ 3KC^{1/2}_\mu \underset{k=0,\dots,K-1}{\max} \sqrt{\epsilon_B(\cF_k,\cF_{k+1})}$ with sample complexity 
      \[n=\tilde{\mathcal{O}}\left(\frac{ R^2C^{3}_\mu\|Q^{\pi_\star}\|_{2,\mu}^4\log (N^2_{\epsilon}/\delta)}{\epsilon^6} \right),\]
      where $\tilde{\mathcal{O}}$ ignores all logarithmic factors except the logarithmic dependence on the covering number $N_{\epsilon}$ defined as
      \[
N_\epsilon=\max_{k=1,...,K}N_{k,\epsilon},\qquad
N_{k,\epsilon}=\cN\big(\tfrac{\epsilon^4}{10^6kRC^2_\mu \|Q^{\pi_\star}\|^2_{2,\mu}}; \cF_k, \infn{\cdot}\big),\quad\text{for }k=1,\dots,K.
      \]
      Alternatively assume Assumptions \ref{assum_bell_opt}, \ref{assum_ex_arg}, \ref{assum_star_fun}, \ref{converage_fut}, \ref{assum_inc_fun}, and \ref{iid_data} (Bellman optimality equation, existence of argmin, star-shaped function space, uniform future state distribution, increasing function range, IID dataset).
   Let $\mu$ be the distribution generating the dataset and $\rho$ be an arbitrary distribution on $\cS\times \cA$.
         Let $\epsilon>0$ and $\delta>0$. With probability $1-\delta$, the policy error of Anc-F-QI with $\lambda_k=\frac{k}{k+2}$ and $ K=\lceil 18C^{1/2}_{\mu,\rho}\|Q^{\pi_\star}\|_{2,\mu}/\epsilon\rceil$,
      satisfies  $ \| g^{\pi_\star}-g^{\pi_K}\|_{2,\rho}
       \le \epsilon+ 3KC^{1/2}_{\mu,\rho} \underset{k=0,\dots,K-1}{\max} \sqrt{\epsilon_B(\cF_k,\cF_{k+1})}$ with sample complexity
      \[n=\tilde{\mathcal{O}}\left(\frac{ R^2C^{3}_{\mu,\rho}\|Q^{\pi_\star}\|_{2,\mu}^4\log (N^2_{\epsilon}/\delta)}{\epsilon^6} \right),\]
      where $\tilde{\mathcal{O}}$ ignores all logarithmic factors except the logarithmic dependence on the covering number $N_{\epsilon}$ defined as
      \[
N_\epsilon=\max_{k=1,...,K}N_{k,\epsilon},\qquad
N_{k,\epsilon}=\cN \big(\tfrac{\epsilon^4}{10^6kRC^2_{\mu,\rho}\|Q^{\pi_\star}\|^2_{2,\mu}}; \cF_k, \infn{\cdot}\big),\quad\text{for }k=1,\dots,K
      \]
\end{theorem}
In the Appendix~\ref{app:sec_4}, we show the full sample complexity with the logarithmic factors.

Under the additional assumption of Bellman completeness ($\epsilon_B =0$), this theorem guarantee that Anc-F-QI produces an $\epsilon$-optimal policy with $\tilde{\mathcal{O}}(1/\epsilon^6)$ sample complexity. To the best of our knowledge, this is the first sample complexity result only assuming the Bellman optimality equation or a weakly communicating MDP.
In Section~\ref{sec::rel_anc}, we improve this sample complexity to $\tilde{\mathcal{O}}(1/\epsilon^4)$ using the relative normalization mechanism.

% single-policy Bellman condition as we clarified in table. Therefore, our theorem could be applied to weakly communicating MDP with finte state space while other prior works do not. 

% Finally we note that we present sample complexity with inherent Bellman error in Appendix. 

% Note that previous covering number result can be transformed to pseudo-dimension. 
\subsection{Single-trajectory dataset}
In this subsection, we study sample complexity with single-trajectory dataset.
% in which depency between data should be considered in contrast to IID dataset.
\begin{assump}[single-trajectory dataset]\label{traj_data}
For given behavior policy $\pi_b$ and initial distribution $\nu$ on $\cS$, dataset is $D = \{s_i, a_i,r_i\}^n_{i=1}$ where $s_1 \sim \nu$, $a_i \sim \pi_b(\cdot \,|\, s_i), s_{i+1} \sim P(\cdot \,|\, s_i,a_i)$.    
\end{assump}
The main technical challenge with single-trajectory data is handling the dependency between samples. Following \cite{antos2008learning, antos2007fitted}, we introduce the following $\beta$-mixing condition ensuring that samples are sufficiently representative and rapidly mixing. 
\begin{definition}[$\beta$-mixing]
    Let $\{Z_t\}^\infty_{t=1}$ be a stochastic process. Denote by $Z^{1:t}$ the collection of $(Z_1, \dots, Z_t)$  where we allowed $t=\infty$. Let $\sigma(Z^{i:j})$ denote the $\sigma$-algebra generated by $Z^{i:j} (i \le j)$. The $m$-th $\beta$-mixing coefficient of $\{Z_t\}$ is defined as
    \[\beta_m= \sup_{t\ge 1} \expec \left[ \sup_{B \in \sigma(Z^{t+m:\infty}) } \big|P(B\, |\, Z^{1:t})-P(B)\big| \right].\]
    $\{Z_t\}$ is said to be $\beta$-mixing if $\beta_m \rightarrow 0$ as $m \rightarrow \infty$. In particular, we say that a $\beta$-mixing process mixes at exponent rate with parameters $\bar{\beta}, b, \kappa>0$ if $\beta_m \le \bar{\beta}exp(-bm^{\kappa})$ holds for all $m \ge 0$.
\end{definition}
% This $\beta$-mixing, which is one of the weakest mixing concepts. 
Roughly speaking, the $\beta$-mixing condition ensures that future samples depend weakly on the past samples. We assume that our single-trajectory is $\beta$-mixing and the distribution is in a steady state, following \cite{antos2008learning, antos2007fitted}.
\begin{assump}[$\beta$-mixing single-trajectory]\label{traj_beta_data}
For single-trajectory dataset $\{s_i,a_i,r_i\}^n_{i=1}$, assume that $s_i$ is strictly stationary with $s_i \sim \nu$
and $\beta$-mixing at exponent rate with parameters $\bar{\beta}, b, \kappa>0$. 
% with the actual rate given by the parameters $\bar{\beta}, b, \kappa>0$. 
% We further assume that sampling policy $\pi_b$ satisfies $\pi_0=\inf_{(s,a) \in \cS \times \cA} \pi_b(s \,|\, a)>0$.    
\end{assump}
% Major technical difficulty in single-trajectory dataset is that we have to deal with dependent samples. The main condition here is that the trajectory should be sufficiently representative and
% rapidly mixing. We also require that the states in the trajectory follow a stationary distribution
% /Then data distribution $\mu$ in single-trajectory is simply defined as $\mu(s,a)= \nu(s)\pi_b(a \,|\, s) $.

Again, following \cite{antos2008learning, antos2007fitted}, as measurement of the capacity of function space, we use pseudo dimension which has been widely studied for complexity analyses with various function classes \cite{anthony2009neural, wainwright2019high}. 
\begin{definition}[pseudo dimension] For a given function class $\cF$ of binary-valued functions, we say the set $x^n_1=(x_1,\dots, x_n)$ is shattered by $\cF$ if cardinality of $\{(f(x_1),\dots,f(x_n)) : f \in \mathcal{F} \}$ is $2^n$. The VC-dimension $V_{\cF}$ of $\cF$ is defined as the largest integer $n$ such that there exist the set $x^n_1$ shattered by $\cF$.
For a given class $\cF$ of real-valued functions, the pseudo-dimension $V_{\cF}$ of is defined as the VC-dimension of the set of indicator function of the subgraphs of functions in $\cF$.  
\end{definition}

We now present lemma which bounds approximation error of of Anc-F-QI for single-trajectory dataset.

\begin{lemma}\label{appx_bel_traj}
Assume Assumptions \ref{assum_bell_opt}, \ref{assum_ex_arg},  \ref{assum_star_fun}, \ref{assum_inc_fun}, \ref{traj_data}, and \ref{traj_beta_data} (Bellman optimality equation, existence of argmin, star-shaped function space, increasing function range, single-trajectory dataset, $\beta$-mixing single-trajectory). Let $\mu$ be the distribution generating the dataset defined as $\mu(s,a)= \nu(s)\pi_b(a \,|\, s) $. Let $\epsilon>0$ and $\delta>0$. 
% Assumption \ref{cvx_cpt_fc},\ref{bellman_err}, \ref{iid_data},     
With probability $1-\delta$, $\{f_k, \hat{T}f_k\}^{K-1}_{k=0}$ of Anc-F-QI with $\lambda_k=\frac{k}{k+2}$ satisfies
\[\|Tf_k-\hat{T}f_k\|^2_{\mu,2} \le \sqrt{\frac{c_{0,k}(\max \{c_{0,k}/b, 1\})^{1/\kappa}}{c_{2,k}n}}+\epsilon_B(\cF_k,\cF_{k+1}),\]
where $c_{0,k}=  (V_{\cF_{k+1}}+V_{(\cF_{k})_{max}})\log n/2 +\log (e/(K\delta))+\log(\max (c_{1,k},\bar{\beta})), c_{1,k}=16e^2(V_{\cF_{k+1}} +1)(V_{(\cF_{k})_{max}} +1) (24e)^{V_{\cF_{k+1}}+V_{(\cF_{k})_{max}}},c_{2,k}=\frac{1}{512(2k+3)^4R^4},  V_{(\cF_{k})_{max}}= 2|\cA|V_{\cF_k} \log(3|\cA|)$.
%     With probability $1-\delta$, we have
% \[\|Tf_k-\hat{T}f_k\|^2_{\mu,2} \le 3\epsilon^2_B+ \sqrt{\frac{42( R+2f_{max})^2 \ln (V_{\cF_k})/\delta)}{n}}\]
\end{lemma}
% \begin{lemma}\label{appx_bel_traj}
% Assume Assumptions \ref{assum_bell_opt}, \ref{assum_ex_arg}, \ref{assum_star_fun}, \ref{assum_inc_fun}, \ref{traj_data}, and \ref{traj_beta_data} (Bellman optimality equation, existence of argmin, star-shaped function space, increasing function range, normalized function space, single-trajectory dataset, $\beta$-mixing single-trajectory).
% % Assumption \ref{cvx_cpt_fc},\ref{bellman_err}, \ref{iid_data},     
% Let $\epsilon>0$ and $\delta>0$. 
% With probability $1-\delta$, for given $\{f_k\}^{K-1}_{k=0}$ of Anc-F-QI, we have
% \[\|Tf_k-\hat{T}f_k\|^2_{\mu,2} \le \epsilon_B+\tilde{\mathcal{O}}\left(1/n^{1/4}\right).\]
% %     With probability $1-\delta$, we have
% % \[\|Tf_k-\hat{T}f_k\|^2_{\mu,2} \le 3\epsilon^2_B+ \sqrt{\frac{42( R+2f_{max})^2 \ln (V_\cF)/\delta)}{n}}\]
% XXX This $\tilde{\mathcal{O}}$ denote only dependency of $n$ and ignore logarithmic factors of $n$XX
% \end{lemma}
We defer the proofs to Appendix~\ref{app:sec_4}, but we quickly note that the proof strategy closely follow \cite{antos2008learning, antos2007fitted} and relies on the Hoeffding inequality under a mixing condition. 

Rougbly speaking, Lemma~\ref{appx_bel_iid} tells that the square of approximation error of the Bellman operator decreases at a $1/\sqrt{n}$ rate with respect to number of sample. Combining Theorem \ref{prop::1} and Lemma \ref{appx_bel_traj}, we obtain following sample complexity result of Anc-F-QI with single-trajectory dataset.

\begin{theorem}\label{Sam_com_traj}
Assume Assumptions \ref{assum_bell_opt}, \ref{assum_ex_arg}, \ref{assum_star_fun}, \ref{converage_tran}, \ref{assum_inc_fun}, \ref{traj_data}, and \ref{traj_beta_data} (Bellman optimality equation, existence of argmin, star-shaped function space, uniform stochastic transition, increasing function range, single-trajectory dataset, $\beta$-mixing single-trajectory). Let $\mu$ be the distribution generating the dataset defined as $\mu(s,a)= \nu(s)\pi_b(a \,|\, s) $. Let $\epsilon>0$ and $\delta>0$. With $1-\delta$ probability, the policy error of Anc-F-QI with $\lambda_k=\frac{k}{k+2}$ and $ K=\lceil9C^{1/2}_\mu\|Q^{\pi_\star}\|_{2,\mu}/\epsilon\rceil$  satisfies  $ \| g^{\pi_\star}-g^{\pi_K}\|_{\infty}
    \le \epsilon+ KC^{1/2}_\mu \underset{k=0,\dots,K-1}{\max} \sqrt{\epsilon_B(\cF_k,\cF_{k+1})}$ with sample complexity       \[n= \tilde{\mathcal{O}}\left(1/\epsilon^{12}\right),\]    where $\tilde{\mathcal{O}}$ only shows the dependence on $\epsilon$.    
    % denote only dependency of $n$ and ignore logarithmic factors of $n$XX
Alternatively, Assume Assumptions \ref{assum_bell_opt}, \ref{assum_ex_arg}, \ref{assum_star_fun}, \ref{converage_fut}, \ref{assum_inc_fun}, \ref{traj_data}, and \ref{traj_beta_data} (Bellman optimality equation, existence of argmin, star-shaped function space, uniform future state distribution, increasing function range, single-trajectory dataset, $\beta$-mixing single-trajectory).
Let $\mu$ be the distribution generating the dataset defined as $\mu(s,a)= \nu(s)\pi_b(a \,|\, s) $ and $\rho$ be an arbitrary distribution on $\cS\times \cA$. Let $\epsilon>0$ and $\delta>0$. With $1-\delta$ probability, the policy error of Anc-F-QI with $\lambda_k=\frac{k}{k+2}$ and  $ K=\lceil 9C^{1/2}_{\mu,\rho}\|Q^{\pi_\star}\|_{2,\mu}/\epsilon\rceil$ satisfies  $ \| g^{\pi_\star}-g^{\pi_K}\|_{2,\rho}
       \le \epsilon+ KC^{1/2}_{\mu,\rho} \underset{k=0,\dots,K-1}{\max} \sqrt{\epsilon_B(\cF_k,\cF_{k+1})}$ with sample complexity
\[n= \tilde{\mathcal{O}}\left(1/\epsilon^{12}\right),\]
where $\tilde{\mathcal{O}}$ only shows the dependence on $\epsilon$.    
\end{theorem}
In the Appendix~\ref{app:sec_4}, we show the full sample complexity with all of the factors.

To the best of our knowledge, this is the first sample complexity result with single-trajectory data in the average-reward setup.
 In Section \ref{sec::rel_anc}, we improve this sample complexity to $\tilde{\mathcal{O}}(\epsilon^{-8})$ using the relative normalization mechanism

\newpage

\section{Relative Anchored Fitted Q-Iteration}\label{sec::rel_anc}
% In discounted reward setup, range of function space is usually set to $\infn{Q_\gamma^\star}$ 
% % $\cF = \{f : \cS \times \cA \rightarrow [- Q_\gamma^\star,  Q_\gamma^\star] \,|\, f \in B(S\times A)\}$ 
% where $Q_\gamma^\star$ is optimal state-action function \cite{munos2008finite, chen2019information}. This is because, by direct computation, we can see if $\infn{f} \le \infn{Q_\gamma^\star}$, $\infn{Tf} \le \infn{Q_\gamma^\star}$ for $f: \cS \times \cA \rightarrow \real$. In average-reward setup, however, this boundedness property does not hold, and notably, 
In this section, we propose \emph{Relative Anchored Fitted Q-Iteration (R-Anc-F-QI)} and improve the sample complexity.
We are motivated by the classical \emph{relative value iteration} \cite{white1963dynamic}. In the tabular case, it is known that standard VI diverges in the average-reward setup \cite[Theorem~9.4.1]{10.5555/528623}, and relative value iteration normalizes the divergent vectors \cite[Section 8.5.5]{10.5555/528623}. In the case of (Anchored) Fitted Q-Iteration, this normalization allows the $f_k$ functions to be bounded and removes the inefficiency associated with the increasing function classes described in Section~\ref{ss:f-range}.

% subtracts some uniform constant vector at each iteration. Motivated by this technique, we consider relative version of Anchored Fitted Q-Iteration to bound our range of function space. (Note that in previous section, range of increasing function range is proportional to number of iteration.) As Theorems show, this will lead us to efficient sample complexity.

% \subsection{Relative Anchored Fitted Q-Iteration}
% and substracting uniform function does not effect on greedy policy since $T(c\mathbf{1}+x)=c\mathbf{1}+T(x)$ and adding uniform function does not effect on max operator.
\begin{algorithm}
\caption{Relative Anchored Fitted Q-Iteration $(D, K, \cF, \{\lambda_{i}\}^K_{i=1})$}\label{R-Anc-F-QI}
\begin{algorithmic}
\State \textbf{Input}: $D = \{s_i, a_i,r_i, s'_i\}^n_{i=1}$, $f_0=0 $, $K\ge 1$, $\{\lambda_i\}^K_{i=1} \subset (0,1)$
      \For{$k=0,1,\dots, K-1$}
      \State $\hat{T}f_{k} =\argmin_{f \in \cF} \sum^{n}_{i=1} \big(f(s_i, a_i)-r_i-\max_{a\in\cA} f_k(s'_i, a)\big)^2$ 
      \State $f_{k+1}=  (1- \lambda_{k+1})f_0+\lambda_{k+1} (\hat{T}f_{k}-\frac{\max \hat{T}f_{k} + \min \hat{T}f_{k} }{2}\mathbf{1})$ 
      \EndFor
      \State $\pi(a \,|\, s) = \argmax_{a\in \cA} f_K(s,a)$
      \State \textbf{Output} $\pi, f_K$
\end{algorithmic}
\end{algorithm}
% \begin{assump}
%  For any $|c|<B $, if f $ $ $c\mathbf{1} \in \cF. $
% \end{assump}

The only difference with  Anchored Fitted Q-Iteration is the subtraction of $\frac{\max \hat{T}f_{k} + \min \hat{T}f_{k} }{2}\mathbf{1}$ in the second line of the for-loop. By direct calculation, we can check that $ \big\|f-\frac{\max f + \min f }{2}\mathbf{1}\big\|_\infty \le \infn{f}$ and subtracting a uniform constant does not effect on greedy policy due the fact that the Bellman operator satisfies $T(c\mathbf{1}+x)=c\mathbf{1}+T(x)$. Thus, we can still apply Proposition~\ref{prop::1} to Relative Anchored Fitted Q-Iteration. 
% we elaborated this point in Appendix. 
 % $T(c\mathbf{1}+x)=c\mathbf{1}+T(x)$

 \begin{assump}[normalized function space]\label{assum_nor_fun}
 If  $ f \in \cF$, $f- \frac{\max f + \min f }{2} \mathbf{1} \in \cF$.
\end{assump}
This assumption ensures that the normalization operation is well ldefined.

 \begin{assump}[range of function space] \label{assum_ran_fun}
  $\cF \subset \{f : \cS \times \cA \rightarrow [- 2\infn{Q^{\pi_\star}},  2\infn{Q^{\pi_\star}}] \,|\, f \in B(S\times A)\}$, where $Q^{\pi_\star}$ is solution of Bellman optimality equation. 
  \end{assump}
% Note that assumption on rage function space is also proper for realizatbilty condition $Q^{\pi_\star} \in\cF$.
Now, unlike increasing function range used for the non-relative Anchored Fitted Q-Iteration, we now have a function space bounded by $Q^{\pi_\star}$. This difference leads to improved efficiency as the following sample complexity results show.

% \subsection{Tables}
\begin{theorem}\label{Sam_com_iid_rel}
Assume Assumptions \ref{assum_bell_opt}, \ref{assum_ex_arg}, \ref{assum_star_fun}, \ref{converage_tran}, \ref{iid_data}, \ref{assum_nor_fun}, and \ref{assum_ran_fun} (Bellman optimality equation, existence of argmin, star-shaped function space, uniform stochastic transition, normalized function space, range of function space, IID dataset).   Let $\mu$ be the distribution generating the dataset. Let $\epsilon>0$ and $\delta>0$. With probability $1-\delta$, the  policy error of R-Anc-F-QI with $\lambda_k=\frac{k}{k+2}$ and $K=\lceil18C^{1/2}_\mu\|Q^{\pi_\star}\|_{2,\mu}/\epsilon\rceil$ satisfies  $ \| g^{\pi_\star}-g^{\pi_K}\|_{\infty}
       \le \epsilon+ 3KC^{1/2}_\mu \sqrt{\epsilon_B(
       \cF,\cF
       )}$ with sample complexity 
      \[n=\tilde{\mathcal{O}}\left(\frac{ (R+\infn{Q^{\pi_\star}})^2\infn{Q^{\pi_\star}}^2C^{3}_\mu\log (N^2_{\epsilon}/\delta)}{\epsilon^4} \right),\]
      where $\tilde{\mathcal{O}}$ ignores all logarithmic factors except the logarithmic dependence on the covering number $N_{\epsilon}$ defined as
      \[
N_{\epsilon}=\cN \big(\tfrac{\epsilon^4}{10^6C^2_\mu(R+\infn{Q^{\pi_\star}})\infn{Q^{\pi_\star}}^2}; \cF, \infn{\cdot}\big).
      \]
      Alternatively, assume Assumptions \ref{assum_bell_opt}, \ref{assum_ex_arg}, \ref{assum_star_fun}, \ref{converage_fut}, \ref{iid_data}, \ref{assum_nor_fun}, and \ref{assum_ran_fun} (Bellman optimality equation, existence of argmin, star-shaped function space, uniform future state distribution, normalized function space, range of function space, IID dataset)
 Let $\mu$ be the distribution generating the dataset and $\rho$ be an arbitrary distribution on $\cS\times \cA$.
 Let $\epsilon>0$ and $\delta>0$. With probability $1-\delta$, the policy error of R-Anc-F-QI with $\lambda_k=\frac{k}{k+2}$ and $ K=\lceil18C^{1/2}_{\mu,\rho}\|Q^{\pi_\star}\|_{2,\mu}/\epsilon\rceil$ satisfies   $ \| g^{\pi_\star}-g^{\pi_K}\|_{2,\rho}
       \le \epsilon+ 3KC^{1/2}_{\mu,\rho} \sqrt{\epsilon_B(
       \cF,\cF
       )}$ with sample complexity
      \[n=\tilde{\mathcal{O}}\left(\frac{ (R+\infn{Q^{\pi_\star}})^2\infn{Q^{\pi_\star}}^2C^{3}_{\mu,\rho}\log (N^2_{\epsilon}/\delta)}{\epsilon^4} \right),\]
      where $\tilde{\mathcal{O}}$ ignores all logarithmic factors except the logarithmic dependence on the covering number $N_{\epsilon}$ defined as
      \[N_{\epsilon}=\cN\big(\tfrac{\epsilon^4}{10^6C^2_{\mu,\rho}(R+\infn{Q^{\pi_\star}})\infn{Q^{\pi_\star}}^2}; \cF, \infn{\cdot}\big).\]
\end{theorem}
\begin{theorem}\label{Sam_com_traj_rel}
    Assume  Assumptions \ref{assum_bell_opt}, \ref{assum_ex_arg}, \ref{assum_star_fun}, \ref{converage_tran}, \ref{traj_data}, \ref{traj_beta_data}, \ref{assum_nor_fun}, and \ref{assum_ran_fun} (Bellman optimality equation, existence of argmin, star-shaped function space, uniform stochastic transition, normalized function space, range of function space, single-trajectory dataset, $\beta$-mixing single-trajectory). Let $\mu$ be the distribution generating the dataset defined as $\mu(s,a)= \nu(s)\pi_b(a \,|\, s) $. Let $\epsilon>0$ and $\delta>0$.
    With probability $1-\delta$, the policy error of Anc-F-QI with $\lambda_k=\frac{k}{k+2}$ and $ K=\lceil9C^{1/2}_\mu\|Q^{\pi_\star}\|_{2,\mu}/\epsilon\rceil$ satisfies  $ \| g^{\pi_\star}-g^{\pi_K}\|_{\infty}
       \le \epsilon+ KC^{1/2}_\mu \sqrt{\epsilon_B(
       \cF,\cF
       )}$ with sample complexity       \[n= \tilde{\mathcal{O}}\left(1/\epsilon^{8}\right),\]    
where $\tilde{\mathcal{O}}$ only shows the dependence on $\epsilon$.    
Alternatively, assume Assumptions \ref{assum_bell_opt}, \ref{assum_ex_arg}, \ref{assum_star_fun}, \ref{converage_fut}, \ref{traj_data}, \ref{traj_beta_data}, \ref{assum_nor_fun}, and \ref{assum_ran_fun} (Bellman optimality equation, existence of argmin, star-shaped function space, uniform future state distribution, normalized function space, range of function space, single-trajectory dataset, $\beta$-mixing single-trajectory). Let $\mu$ be the distribution generating the dataset defined as $\mu(s,a)= \nu(s)\pi_b(a \,|\, s) $ and $\rho$ be an arbitrary distribution on $\cS\times \cA$. Let $\epsilon>0$ and $\delta>0$.
    With probability $1-\delta$, the policy error of Anc-F-QI with $\lambda_k=\frac{k}{k+2}$ and $ K=\lceil9C^{1/2}_{\mu,\rho}\|Q^{\pi_\star}\|_{2,\mu}/\epsilon\rceil$ satisfies  $ \| g^{\pi_\star}-g^{\pi_K}\|_{2,\rho}
       \le \epsilon+ KC^{1/2}_{\mu,\rho} \sqrt{\epsilon_B(
       \cF,\cF
       )}$ with sample complexity
\[n= \tilde{\mathcal{O}}\left(1/\epsilon^{8}\right),\]    
where $\tilde{\mathcal{O}}$ only shows the dependence on $\epsilon$.    
\end{theorem}

Indeed, with the relative normalization mechanism, we improve the sample complexities from $\tilde{\mathcal{O}}(1/\epsilon^{6})$ to $\tilde{\mathcal{O}}(1/\epsilon^{4})$ and $\tilde{\mathcal{O}}(1/\epsilon^{12})$ to $\tilde{\mathcal{O}}(1/\epsilon^{8})$ for IID and single-trajectory data cases, respectively,

\section{Conclusion}
In this work, we introduced Anchored Fitted Q-Iteration (Anc-F-QI) and established new sample complexity results for the average-reward offline RL with general function approximation under the assumption of weakly communicating MDPs. Our approach combines the classical Fitted Q-Iteration with an anchoring mechanism, and the anchor mechanism is the crucial component that enables the finite-time analysis. Roughly speaking, we establish a $\tilde{\mathcal{O}}(1/\epsilon^6)$ sample complexity with IID data and $\tilde{\mathcal{O}}(1/\epsilon^{12})$ sample complexity with single-trajectory data. Then, using the relative normalization technique, we improve the sample complexity to $\tilde{\mathcal{O}}(1/\epsilon^4)$ and $\tilde{\mathcal{O}}(1/\epsilon^8)$ for IID and single-trajectory data, respectively.

% First, value-iteration type algorithms (others two are primal dual) 
% Anchoring mechiasim well incorporates with prior approach in discounted reward since compared to value iteration and fitted Q-Iteration, by adding anchoring process, we can obtain valid convergence. We believe our results are necessary in the Average reward MDP since many techinqeus developed for value iteration, we believe it can be also applied to Anchored value iteration in average reward setup as we did in this work and this is meaningful result in average-reward setup.

One limitation of this work is the reliance on \emph{full} coverage coefficients as described in Assumptions \ref{converage_tran} and \ref{converage_fut}. Some prior work, such as \cite{ozdaglar2023revisiting} and \cite{gabbianelli2024offline}, utilizes a weaker assumption that we refer to as \emph{partial} coverage coefficients, albeit with much stronger structural assumptions on the MDP. Extending our analysis to relax the full coverage coefficient would be a worthwhile direction of future work. Another possible direction of future work is to utilize variance reduction techniques in the style of \cite{wainwright2019variance,sidford2023variance,lee2025near} to further improve the sample complexity.

% All tables must be centered, neat, clean and legible.  The table number and
% title always appear before the table.  See Table~\ref{sample-table}.

% Place one line space before the table title, one line space after the
% table title, and one line space after the table. The table title must
% be lower case (except for first word and proper nouns); tables are
% numbered consecutively.

% Note that publication-quality tables \emph{do not contain vertical rules.} We
% strongly suggest the use of the \verb+booktabs+ package, which allows for
% typesetting high-quality, professional tables:
% \begin{center}
%   \url{https://www.ctan.org/pkg/booktabs}
% \end{center}
% This package was used to typeset Table~\ref{sample-table}.

% \begin{table}
%   \caption{Sample table title}
%   \label{sample-table}
%   \centering
%   \begin{tabular}{lll}
%     \toprule
%     \multicolumn{2}{c}{Part}                   \\
%     \cmidrule(r){1-2}
%     Name     & Description     & Size ($\mu$m) \\
%     \midrule
%     Dendrite & Input terminal  & $\sim$100     \\
%     Axon     & Output terminal & $\sim$10      \\
%     Soma     & Cell body       & up to $10^6$  \\
%     \bottomrule
%   \end{tabular}
% \end{table}

\begin{ack}
This work is supported by the National Research Foundation of Korea (NRF) grant funded by the Korean government (No.RS-2024-00421203).
% Use unnumbered first level headings for the acknowledgments. All acknowledgments
% Do {\bf not} include this section in the anonymized submission, only in the final paper. You can use the \texttt{ack} environment provided in the style file to automatically hide this section in the anonymized submission.
\end{ack}

\newpage
\bibliographystyle{abbrv}
\bibliography{neurips_2025}

%%%%%%%%%%%%%%%%%%%%%%%%%%%%%%%%%%%%%%%%%%%%%%%%%%%%%%%%%%%%

\newpage 

\section*{NeurIPS Paper Checklist}

\begin{enumerate}

\item {\bf Claims}
    \item[] Question: Do the main claims made in the abstract and introduction accurately reflect the paper's contributions and scope?
    \item[] Answer: \answerYes{}.
    \item[] Justification: Our abstract and introduction clearly state the claims made, including the contributions.
contributions made in the paper

\item {\bf Limitations}
    \item[] Question: Does the paper discuss the limitations of the work performed by the authors?
    \item[] Answer: \answerYes{}.
    \item[] Justification: Yes, we present limitation of our work as table and discuss in conclusion.

\item {\bf Theory assumptions and proofs}
    \item[] Question: For each theoretical result, does the paper provide the full set of assumptions and a complete (and correct) proof?
    \item[] Answer: \answerYes{}.
    \item[] Justification: Yes, we clearly state full set of assumptions and proofs.

    \item {\bf Experimental result reproducibility}
    \item[] Question: Does the paper fully disclose all the information needed to reproduce the main experimental results of the paper to the extent that it affects the main claims and/or conclusions of the paper (regardless of whether the code and data are provided or not)?
    \item[] Answer: \answerNA{}.
    \item[] Justification: Our work does not include numerical experiments.

\item {\bf Open access to data and code}
    \item[] Question: Does the paper provide open access to the data and code, with sufficient instructions to faithfully reproduce the main experimental results, as described in supplemental material?
    \item[] Answer: \answerNA{}.
    \item[] Justification: Our paper does not include experiments requiring code.
    \item[] Guidelines:
    \begin{itemize}
        \item The answer NA means that paper does not include experiments requiring code.
        \item Please see the NeurIPS code and data submission guidelines (\url{https://nips.cc/public/guides/CodeSubmissionPolicy}) for more details.
        \item While we encourage the release of code and data, we understand that this might not be possible, so ``No'' is an acceptable answer. Papers cannot be rejected simply for not including code, unless this is central to the contribution (e.g., for a new open-source benchmark).
        \item The instructions should contain the exact command and environment needed to run to reproduce the results. See the NeurIPS code and data submission guidelines (\url{https://nips.cc/public/guides/CodeSubmissionPolicy}) for more details.
        \item The authors should provide instructions on data access and preparation, including how to access the raw data, preprocessed data, intermediate data, and generated data, etc.
        \item The authors should provide scripts to reproduce all experimental results for the new proposed method and baselines. If only a subset of experiments are reproducible, they should state which ones are omitted from the script and why.
        \item At submission time, to preserve anonymity, the authors should release anonymized versions (if applicable).
        \item Providing as much information as possible in supplemental material (appended to the paper) is recommended, but including URLs to data and code is permitted.
    \end{itemize}

\item {\bf Experimental setting/details}
    \item[] Question: Does the paper specify all the training and test details (e.g., data splits, hyperparameters, how they were chosen, type of optimizer, etc.) necessary to understand the results?
    \item[] Answer: \answerNA{}.
    \item[] Justification: Our paper does not include numerical experiments.

\item {\bf Experiment statistical significance}
    \item[] Question: Does the paper report error bars suitably and correctly defined or other appropriate information about the statistical significance of the experiments?
    \item[] Answer: \answerNA{}.
    \item[] Justification: Our paper does not include numerical experiments.

\item {\bf Experiments compute resources}
    \item[] Question: For each experiment, does the paper provide sufficient information on the computer resources (type of compute workers, memory, time of execution) needed to reproduce the experiments?
    \item[] Answer: \answerNA{}.
    \item[] Justification: Ou paper does not include numerical experiments.
    
\item {\bf Code of ethics}
    \item[] Question: Does the research conducted in the paper conform, in every respect, with the NeurIPS Code of Ethics \url{https://neurips.cc/public/EthicsGuidelines}?
    \item[] Answer: \answerYes{}.
    \item[] Justification: Our paper conforms, in every respect, with the
NeurIPS Code of Ethic.

\item {\bf Broader impacts}
    \item[] Question: Does the paper discuss both potential positive societal impacts and negative societal impacts of the work performed?
    \item[] Answer: \answerNA{}.
    \item[] Justification: Since our work is a theory paper, there is no societal impact of the work performed.
    
\item {\bf Safeguards}
    \item[] Question: Does the paper describe safeguards that have been put in place for responsible release of data or models that have a high risk for misuse (e.g., pretrained language models, image generators, or scraped datasets)?
    \item[] Answer: \answerNA{}.
    \item[] Justification: Our paper poses no such risks.

\item {\bf Licenses for existing assets}
    \item[] Question: Are the creators or original owners of assets (e.g., code, data, models), used in the paper, properly credited and are the license and terms of use explicitly mentioned and properly respected?
    \item[] Answer: \answerNA{}.
    \item[] Justification: Our paper does not use existing assets.

\item {\bf New assets}
    \item[] Question: Are new assets introduced in the paper well documented and is the documentation provided alongside the assets?
    \item[] Answer: \answerNA{}.
    \item[] Justification: Our paper does not release new assets.

\item {\bf Crowdsourcing and research with human subjects}
    \item[] Question: For crowdsourcing experiments and research with human subjects, does the paper include the full text of instructions given to participants and screenshots, if applicable, as well as details about compensation (if any)? 
    \item[] Answer: \answerNA{}.
    \item[] Justification: Our paper does not involve crowdsourcing nor research with human subjects.

\item {\bf Institutional review board (IRB) approvals or equivalent for research with human subjects}
    \item[] Question: Does the paper describe potential risks incurred by study participants, whether such risks were disclosed to the subjects, and whether Institutional Review Board (IRB) approvals (or an equivalent approval/review based on the requirements of your country or institution) were obtained?
    \item[] Answer: \answerNA{}.
    \item[] Justification: Our paper does not involve crowdsourcing nor research with human subjects.

\item {\bf Declaration of LLM usage}
    \item[] Question: Does the paper describe the usage of LLMs if it is an important, original, or non-standard component of the core methods in this research? Note that if the LLM is used only for writing, editing, or formatting purposes and does not impact the core methodology, scientific rigorousness, or originality of the research, declaration is not required.
    %this research? 
    \item[] Answer: \answerNA{}.
    \item[] Justification: Our paper does not involve LLMs as any important, original, or non-standard components. 

\end{enumerate}

\newpage

\appendix

\section{Prior works}

\paragraph{Average-Reward MDP}
The setup of average reward MDPs was introduced  in the dynamic programming literature by \cite{howard1960dynamic}, and \cite{blackwell1962discrete} established a theoretical framework for their analysis. In reinforcement learning (RL), average-reward MDP was mainly considered in the sample-based setup where the transition matrix and reward are unknown \cite{mahadevan1996average, dewanto2020average}. For this setup, various methods were proposed: model-based methods \cite{jin2021towards, zurek2024span}, Q-learning methods \cite{wei2020model, wan2021learning}, and policy gradient methods \cite{bai2024regret,murthy2023convergence, kumar2024global}. Sample complexity to obtain $\epsilon$-optimal under generative model \cite{wang2017primal, zhang2023sharper, lee2025near, li2024stochastic, jin2024feasible} and for regret minimization  \cite{burnetas1997optimal, Jaksch2010, zhang2019regret,boone2024achieving} also have been actively studied. 

% XXXAlthough many prior analsyes based on constraints on MDP such that to release constraints on MDPXXX In this setup, constraints on MDP imporatnt facotr relaeses the analsyes, SO mosts analyses based on ergodic assumption or mixing condtion  Recently, workst studyXX It was known that in average reward, is imporatnt issue in analyese.

% For this setup, analyze regret minimization problem for unichain and communicating MDPs.

\paragraph{Value Iteration}
Value iteration (VI) was first introduced in the dynamic programming literature \cite{bellman1957markovian} and serve as a fundamental algorithm to compute the value functions. The sample-based variants, such as TD-Learning~\cite{Sutton1988}, Fitted Value Iteration~\cite{ernst2005tree,munos2008finite}, and Deep Q-Network~\cite{MnihKavukcuogluSilveretal2015} are the workhorses of modern reinforcement learning algorithms~\cite{Bertsekas96,sutton2018reinforcement,SzepesvariBook10}. VI is also routinely applied in diverse settings, including factored MDPs \cite{rosenberg2021oracle}, robust MDPs \cite{kumarefficient}, MDPs with reward machines \cite{bourel2023exploration}, MDPs with options \cite{fruit2017regret}, and generative model  \cite{wainwright2019variance, sidford2023variance, lee2025near}.

 The convergence of VI in average-reward MDPs also has been extensively studied. For unichain MDPs, delta coefficient, ergodicity coefficient, and the J-stage span contraction demonstrate the linear rate of VI \citep{seneta2006non, hubner1977improved,federgruen1978contraction, van1981stochastic}. When MDP is multichain, it is known that policy error of VI might not converge to zero \citep[Example~4]{della2012illustrated}. Even with the aperiodicity assumption, VI guarantees only asymptotic convergence. \cite [Theorem~9.4.5]{10.5555/528623}. \citep{schweitzer1977asymptotic, schweitzer1979geometric} established necessary and sufficient conditions of convergence of VI and asymptotic linear convergence on Bellman error.  
  
\paragraph{Offline Reinforcement Learning}
In offline RL, the agent learns decision-making strategies utilizing precollected data \cite{levine2020offline}. This framework is often applied when interaction with the environment can be expensive, and the quantities of data that can be gathered online are substantially lower than the precollected dataset \cite{chen2019top, kosorok2019precision, levine2020offline}. Consequently, various offline RL methods have been actively proposed \cite{ernst2005tree, siegel2020keep,kumar2020conservative, agarwal2020optimistic}, and Fitted Q-Iteration is one of the representative methods based on sample-based value iteration with function approximation \cite{ernst2005tree,munos2008finite}.

One issue in offline RL is the distribution mismatch between the behavior policy that collected the data and the learned policy of the agent \cite{kumar2019stabilizing, wang2020statistical}. For theoretical analysis, \emph{coverage coefficient} is assumed to ensure that offline dataset sufficiently explores whole state and action space. \cite{munos2007performance,scherrer2014approximate,uehara2021pessimistic}. Under this assumption, sample complexity of offline RL methods actively analyzed \cite{antos2008learning,ross2012agnostic,chen2019information, ozdaglar2023revisiting}, and in particular, an $L_p$ bound of approximate value iteration was obtained, which in turn yields convergence results for Fitted Q-Iteration \cite{munos2007performance, munos2008finite}. More recently, several works succeeded relaxing the full coverage assumption to partcal coverage \cite{liu2020provably,rashidinejad2021bridging,xie2021bellman,jin2021pessimism}. 

Another issue in offline RL is the representation capacity of the chosen function space. To handle large state space and action spaces, many RL frameworks including offline RL use function approximation, ranging from linear functions \cite{duan2020minimax} and nonlinear (general) functions such as neural networks \cite{fan2020theoretical} and kernel functions \cite{chang2021mitigating}. In offline RL, the \emph{inherent Bellman error} measures the approximation error incurred when projecting the output of Bellman operator into chosen function space, and Bellamn completeness assumes the inherent Bellman error is zero \cite{munos2008finite, chen2019information}. Most sample complexity analyses in offline RL rely on inherent Bellman error or  Bellman completeness assumption \cite{liu2020provably,rashidinejad2021bridging,xie2021bellman,jin2021pessimism}. Recently, however, several works achieved finite sample complexity under weaker realizabiltiy assumption, which only requires that optimal function value lies within chosen function space \cite{xie2021batch,zhan2022offline}. 
% /Also,\cite{antos2007fitted,antos2008learning} studied smaple complexity of offline RL with single-trajectory dataset beyond the IID date set.

 Most of prior works in offline RL focused on discounted-reward setup, and to the best of our knowledge, two prior works established the finite sample complexity in the offline average-reward setup \cite{ozdaglar2024, gabbianelli2024offline}. Both proposed a primal-dual approach, reformulating the Bellman equation as a bilinear saddle-point problem, to obtain an $\epsilon$-optimal policy under partial coverage. However, they imposed restrictive structural assumptions on MDP such as uniform mixing or linearity and considered only IID dataset. (See the Table \ref{table:main}.)  

\newpage
\section{Preliminaries}
The followings are inequalities from prior works used in the proof. 
    \begin{fact}[Bernstein inequality]
    Let $X_1, \dots, X_n$ are indepedent random variables. If $X_i  \le b$ for all $i$,
    then
\[ \mathbb{P}\left(\frac{1}{n}\sum^n_{i=1} X_i- \expec [X_i] \ge \epsilon\right) \le exp \left[-\frac{n^2 \epsilon^2}{2\sum^n_{i=1} \expec{[X^2_i]} +nb\epsilon/3}\right]\]
Furthermore, if all the $\expec{[X_i^2]}$ are equal, with $1-\delta$ probability, 
\[\frac{1}{n}\sum^n_{i=1} X_i -\expec X_i  \le \sqrt{2\expec{[X_1^2]}\ln (1/\delta)/n}+\frac{2b \ln(1/\delta)}{3n} .\]
\end{fact}

\begin{fact}[\cite{antos2008learning}\label{hoeff_mix}, Lemma 4]
    Suppose that $Z_1,\dots, Z_n \in \cZ$ is a stationary $\beta$-mixing process with mixing coefficients $\beta_m$, $Z'_t \in \cZ (t \in H)$ are the block-independent ghost samples. $H = \{2ik_N+j \,:\, 0 \le i< m_n, 1 \le j \le k_N\}$ and $\cF$ is permissible class of  $\cZ \rightarrow [-M,M]$ functions. Then 
    \[P\left(\sup_{f \in \cF} \left|\frac{1}{N} \sum^N_{n=1}f(Z_n)-\expec[f(Z_1)]\right|>\epsilon\right) \le 16 \expec [\cN ( \epsilon/8, \cF, l_{(Z'_t)_{t \in H}})]e^{-\frac{m_N \epsilon^2}{128M^2}}+2m_N \beta_{k_N+1} .\]
\end{fact}

% \begin{prop}
%    If  $\cP_1$ and $\cP_2$ are probablity operaotrs, then $\cP_1\cP_2$ is probability operator and $\infn{\cP_1\cP_2} \le \infn{P_1}\infn{P_2}$.
% \end{prop}
\section{Omitted proofs in Section \ref{sec::Apx-Anc-QI}}
 % For notational simplicity, we write $g^{\pi_\star}$ and $Q^{\pi_\star}$ instead of $g^{\pi_{\star}}$ and $Q^{\pi_{\star}}$.
\subsection{Proof of Proposition \ref{prop::1}}
 Define the limiting matrix $\cP_*^{\pi}$ as the Ces\`aro limit of $\cP^{\pi}$, i.e., $\cP_*^{\pi}=\lim \frac{1}{n}\sum^n_{i=1}(\cP^{\pi})^i$. (The limiting matrix always exists for finite state-action spaces \citep[Appendix A.4]{10.5555/528623}.)  Then, $\cP_*^{\pi}$ is stochastic and, by definition, $g^\pi = \cP^\pi_* r$ \citep[Proposition 8.1.1]{10.5555/528623}. 
 
 We first prove following lemma. 
\begin{lemma}\label{lem::1} 
Let $\lambda_{K+1}=1$. Under Assumption \ref{assum_bell_opt} (Bellman optimality equation), the policy error of \ref{eq:Apx-Anc-QI} satisfies
\begin{align*}
    &g^{\pi_\star}-g^{\pi_K} =\cP^{\pi_K}_{*} (g^{\pi_\star}-TQ^K+Q^K) \\&\le \cP^{\pi_K}_{*}\bigg(\sum^K_{l=0}\Pi^K_{i=l+1}\lambda_{i} (\lambda_{l+1}-\lambda_{l}) \Pi^K_{i=l+1}\cP^{\pi_{i}}\left(\sum^l_{m=0} \Pi^l_{i=m+1}\lambda_{i} 
 (1-\lambda_{m}) (\cP^{\pi_\star})^{l+1-m}-I\right)(Q^0-Q^{\pi_\star})\\&+\sum^K_{l=1}\Pi^K_{i=l}\lambda_{i}\left(\sum^K_{m=l} (\lambda_{m+1}-\lambda_{m}) \Pi^K_{i=m+1}\cP^{\pi_{i}}  
 (\cP^{\pi_\star})^{m+1-l}+  \Pi^K_{i=l+1}\cP^{\pi_{i}} (\lambda_l\cP^{\pi_{l}}-I)\right) \epsilon_l\bigg).
\end{align*}    
% \begin{fact}
% \begin{align*}
%  Q^{\pi_\star} - Q^{\pi^K}
% &\le (I - \gamma \,P^{\pi^K})^{-1}
% \Biggl\{
% \sum_{i=0}^{K-1}
% \gamma^{K-i}\,
% \bigl[(P^{\pi^*})^{K-i}
% - P^{\pi^K}P^{\pi^{K-1}}\cdots P^{\pi^{\,i+1}}\bigr]\,\epsilon_{i+1} \\[6pt]
% &\quad\;+\;\gamma^{K+1}\,
% \bigl[(P^{\pi^*})^{K+1}
% - P^{\pi^K}P^{\pi^{K-1}}\cdots P^{\pi^0}\bigr]
% \bigl(Q^* - Q_0\bigr)
% \Biggr\}.
% \end{align*}
    
% \end{fact}
% \begin{align*}
%     &g^{\pi_\star}-g^{\pi_k}  \le \cP^{\pi_k}_{*}\bigg(\sum^k_{l=0}\Pi^k_{i=l+1}\lambda_{i} (\lambda_{l+1}-\lambda_{l}) \Pi^k_{i=l+1}\cP^{\pi_{i}}\left(\sum^l_{m=0} \Pi^l_{i=m+1}\lambda_{i} 
%  (1-\lambda_{m}) (\cP^{\pi_\star})^{l+1-m}-I\right)(Q^0-Q^{\pi_\star})\\&+\sum^k_{l=1}\left(\sum^k_{m=l}\Pi^k_{i=m+1}\lambda_{i} (\lambda_{m+1}-\lambda_{m}) \Pi^m_{i=l}\lambda_{i}\Pi^k_{i=m+1}\cP^{\pi_{i}}  
%  (\cP^{\pi_\star})^{m+1-l}+\Pi^k_{i=l}\lambda_{i}  \Pi^k_{i=l+1}\cP^{\pi_{i}} (\lambda_l\cP^{\pi_{l}}-I)\right) \epsilon_l\bigg).
% \end{align*}    
% where $\lambda_{k+1}=1$. 
\end{lemma}
\begin{proof}[Proof of Lemma \ref{lem::1}]
    By definition of \ref{eq:Apx-Anc-QI}, we have
\begin{align*}
    &TQ^K-Q^K 
    \\&= (1-\lambda_{K})(TQ^K-Q^0)+\lambda_{K}(TQ^K-TQ^{K-1}) -\lambda_K\epsilon_{K}
    \\& \ge (1-\lambda_{K})(TQ^K-Q^0)+\lambda_{K} \cP^{\pi_{K}}(Q^K-Q^{K-1}) -\lambda_K\epsilon_{K}
     \\& \ge (1-\lambda_{K})(TQ^K-Q^0) -\lambda_K\epsilon_{K}
     \\&+\lambda_{K} \cP^{\pi_{K}}((\lambda_{K}-\lambda_{K-1})(TQ^{K-1}-Q^0)+\lambda_{K-1}(TQ^{K-1}-TQ^{K-2})+\lambda_K\epsilon_{K}-\lambda_{K-1}\epsilon_{K-1})
     \\& \ge \sum^K_{l=0}\Pi^K_{i=l+1}\lambda_{i} (\lambda_{l+1}-\lambda_{l}) \Pi^K_{i=l+1}\cP^{\pi_{i}}(TQ^{l}-Q^0)
+ \sum^K_{l=1}\Pi^K_{i=l}\lambda_{i}  \Pi^K_{i=l+1}\cP^{\pi_{i}} (\lambda_l\cP^{\pi_{l}}-I)\epsilon_{l}
\end{align*}
where first inequality comes from greedy policy and last inequality comes from induction. 

For any $0 \le l \le K$,
\begin{align*}
    &TQ^l-Q^0 \\&= TQ^l-Q^{\pi_\star}-(Q^0-Q^{\pi_\star})\\ 
     & = TQ^l-TQ^{\pi_\star}+g^{\pi_\star}-(Q^0-Q^{\pi_\star})\\
     & \ge \cP^{\pi_\star}(Q^l-Q^{\pi_\star})+g^{\pi_\star}-(Q^0-Q^{\pi_\star})\\
     & = \cP^{\pi_\star}(\lambda_{l}(TQ^{l-1}-Q^{\pi_\star})+(1-\lambda_{l})(Q^0-Q^{\pi_\star})+\lambda_l \epsilon_{l})+g^{\pi_\star}-(Q^0-Q^{\pi_\star})\\
     & \ge \left(\sum^l_{m=0} \Pi^l_{i=m+1}\lambda_{i} 
 (\cP^{\pi_\star})^{l+1-m} (1-\lambda_{m})-I\right)(Q^0-Q^{\pi_\star})+\sum^l_{m=0} \Pi^l_{i=m+1} \lambda_{i}g^{\pi_\star}
 \\& +\sum^l_{m=1} \Pi^l_{i=m}\lambda_{i} 
 (\cP^{\pi_\star})^{l+1-m} \epsilon_m,
\end{align*}
where second equality comes from Bellman optimality equation. By combining previous two inequalities, we get 
\begin{align*}
    &TQ^K-Q^K \\&\ge 
    \sum^K_{l=0}\Pi^K_{i=l+1} \lambda_{i} (\lambda_{l+1}-\lambda_{l}) \Pi^K_{i=l+1}\cP^{\pi_{i}}\sum^l_{m=0} \Pi^l_{i=m+1}\lambda_{i}g^{\pi_\star}
     \\ \qquad &+\sum^K_{l=0}\Pi^K_{i=l+1}\lambda_{i} (\lambda_{l+1}-\lambda_{l}) \Pi^k_{i=l+1}\cP^{\pi_{i}}\left(\sum^l_{m=0} \Pi^l_{i=m+1}\lambda_{i} 
 (\cP^{\pi_\star})^{l+1-m} (1-\lambda_{m})-I\right)(Q^0-Q^{\pi_\star})
 \\& + \sum^K_{l=1}\Pi^K_{i=l}\lambda_{i}  \Pi^K_{i=l+1}\cP^{\pi_{i}} (\lambda_l\cP^{\pi_{l}}-I)\epsilon_{l}
 \\& + \sum^K_{l=1}\sum^l_{m=1}\Pi^K_{i=l+1}\lambda_{i} (\lambda_{l+1}-\lambda_{l}) \Pi^K_{i=l+1}\cP^{\pi_{i}} \Pi^l_{i=m}\lambda_{i} 
 (\cP^{\pi_\star})^{l+1-m} \epsilon_m
 \\& =g^{\pi_\star}  + \sum^K_{l=1}\bigg(\sum^k_{m=l}\Pi^K_{i=m+1}\lambda_{i} (\lambda_{m+1}-\lambda_{m}) \Pi^m_{i=l}\lambda_{i}\Pi^K_{i=m+1}\cP^{\pi_{i}}  
 (\cP^{\pi_\star})^{m+1-l}\\&  +\Pi^K_{i=l}\lambda_{i}  \Pi^K_{i=l+1}\cP^{\pi_{i}} (\lambda_l\cP^{\pi_{l}}-I)\bigg) \epsilon_l
 \\&  +\sum^K_{l=0}\Pi^K_{i=l+1}\lambda_{i} (\lambda_{l+1}-\lambda_{l}) \Pi^K_{i=l+1}\cP^{\pi_{i}}\left(\sum^l_{m=0} \Pi^l_{i=m+1}\lambda_{i} 
 (1-\lambda_{m}) (\cP^{\pi_\star})^{l+1-m}-I\right)(Q^0-Q^{\pi_\star}).
\end{align*}
This implies
\begin{align*}
    &TQ^K-Q^K -g^{\pi_\star} 
    \\& \ge \sum^K_{l=0}\Pi^K_{i=l+1}\lambda_{i} (\lambda_{l+1}-\lambda_{l}) \Pi^K_{i=l+1}\cP^{\pi_{i}}\left(\sum^l_{m=0} \Pi^l_{i=m+1}\lambda_{i} 
 (1-\lambda_{m}) (\cP^{\pi_\star})^{l+1-m}-I\right)(Q^0-Q^{\pi_\star})\\&+\sum^K_{l=1}\Pi^K_{i=l}\lambda_{i}\left(\sum^K_{m=l} (\lambda_{m+1}-\lambda_{m}) \Pi^K_{i=m+1}\cP^{\pi_{i}}  
 (\cP^{\pi_\star})^{m+1-l}+  \Pi^K_{i=l+1}\cP^{\pi_{i}} (\lambda_l\cP^{\pi_{l}}-I)\right) \epsilon_l.
\end{align*}    
Finally, following the proof of \cite[Theorem 8.5.5]{10.5555/528623}, we have
\begin{align*}
    g^{\pi_\star}-g^{\pi_K} = \cP^{\pi_K}_{*} (g^{\pi_\star} -r) &= \cP^{\pi_K}_{*} (g^{\pi_\star} -r-\cP^{\pi_K}Q^K+Q^K)\\& =\cP^{\pi_K}_{*} (g^{\pi_\star} -TQ^K+Q^K),
\end{align*}
where first equality comes from Bellman optimality equation and second equality comes from property of limiting matrix.
% \begin{align*}
%     g^{\pi_\star}-g^{\pi_k} = \cP^{\pi_k}_{*} (g^{\pi_\star} -r^{\pi_k}) &= \cP^{\pi_k}_{*} (g^{\pi_\star} -r^{\pi_k}-\cP^{\pi_k}Q^k(s, \pi_k(s))+Q^k(s, \pi_k(s)))\\& =\cP^{\pi_k}_{*} ((g^{\pi_\star} -TQ^k+Q^k)(s, \pi_k(s)))
% \end{align*}
This implies that 
\begin{align*}
   &g^{\pi_\star}-g^{\pi_K} =\cP^{\pi_K}_{*} (g^{\pi_\star}-TQ^K+Q^K) \\&\le \cP^{\pi_K}_{*}\bigg(\sum^K_{l=0}\Pi^K_{i=l+1}\lambda_{i} (\lambda_{l+1}-\lambda_{l}) \Pi^K_{i=l+1}\cP^{\pi_{i}}\left(\sum^l_{m=0} \Pi^l_{i=m+1}\lambda_{i} 
 (1-\lambda_{m}) (\cP^{\pi_\star})^{l+1-m}-I\right)(Q^0-Q^{\pi_\star})\\&+\sum^K_{l=1}\Pi^K_{i=l}\lambda_{i}\left(\sum^K_{m=l} (\lambda_{m+1}-\lambda_{m}) \Pi^K_{i=m+1}\cP^{\pi_{i}}  
 (\cP^{\pi_\star})^{m+1-l}+  \Pi^K_{i=l+1}\cP^{\pi_{i}} (\lambda_l\cP^{\pi_{l}}-I)\right) \epsilon_l\bigg).
\end{align*}    
\end{proof}
% \subsection{Proof of Proposition \ref{prop::1}}
% \begin{proof}
%     By condition and Lemma \ref{lem::1}, through calculation, we obtain conclusion.
%     % \begin{align*}
%     %     \infn{g^{\pi_\star}-g^{\pi_k}} \le \frac{4}{k+2}\infn{Q^0-Q^{\pi_\star}}+\frac{n+2}{3}\epsilon
%     % \end{align*}
% \end{proof}

The following are lemmas about coverage coefficient $C_{\mu,\rho}$. 
\begin{lemma}\label{lem_conv}
If $\cP_1$ and $\cP_2$ are stochastic matrix satisfying $ \rho^\top \cP_i \le C_{\mu,\rho}\mu$ for $i=1,2$ and given distribution $\mu$ and $\rho$ on $\cS \times \cA$, then $\rho^\top (a\cP_1+ (1-a)\cP_2) \le C_{\mu,\rho}\mu$ for $0 \le a \le 1$.
\end{lemma}
\begin{lemma}\label{lem_opt}
Under Assumption \ref{converage_fut} (uniform future state distribution), 
\[ \sup_{\pi_1,\pi_2, \dots \pi_k} \infn{\frac{ \rho^{\top} \cP_*^{\pi_\star}\cP^{\pi_1}\cP^{\pi_2}\cdots \cP^{\pi_k}(\cdot)}{\mu(\cdot)}} \le C_{\mu,\rho}\]
where $\pi_\star\pi_1,\pi_2, \dots \pi_k$ represents an arbitrary sequence of policies with optimal policy.
\end{lemma}
\begin{proof}
    Under Assumption \ref{converage_fut}, for any non negative integer $n$, we have $ \rho^{\top} (\cP^{\pi_\star})^n\cP^{\pi_1}\cP^{\pi_2}\cdots \cP^{\pi_k}(\cdot) \le C_{\mu,\rho}\mu$. This implies $ \rho^{\top} \cP_*^{\pi_\star}\cP^{\pi_1}\cP^{\pi_2}\cdots \cP^{\pi_k}(\cdot) \le C_{\mu,\rho}\mu$ by definition of limiting matrix.
\end{proof}
% \begin{lemma}
% Under assumption $2$, $\cP^{\pi}_{\star} $ also satisfies.
% \end{lemma}
% \begin{proof}
%     $\rho^\top\cP^{\pi}_{\star}  \prod \cP^{\pi_i} = \lim_n \frac{1}{n} \sum^{n-1}_{j=0}\rho^\top(\cP^{\pi})^j  \prod \cP^\pi_i \le \lim_n \frac{1}{n} \sum^{n-1}\mu =\mu$
% \end{proof}
% \begin{lemma}
%     If $\nu$ is distribution satisfying Assumption 2, then  $\|\nu^\top Q\|_{p, \rho}\le C^{1/p}_{\mu} \|Q\|_{p, \mu} $.
% \end{lemma}
\begin{lemma}\label{lem_cov}
    If $\cP$ is stochastic matrix satisfying $ \rho^\top \cP \le C_{\mu,\rho}\mu^\top$ for given distribution $\mu$ and $\rho$ on $\cS \times \cA$, then  $\|\cP Q\|_{p, \rho}\le C^{1/p}_{\mu} \|Q\|_{p, \mu} $.
\end{lemma}
\begin{proof}
% $\rho(s,a)((PQ)(s,a))^p$ $\expec_{(s,a) \sim \rho}{|\cP Q|^p}$ $(\cP^\pi Q)(s,a) = \expec_{a' \sim \pi(\cdot \,|\, s'), s' \sim P(\cdot \,|\, s,a) }[Q(s',a')]$
    Since $|\cP Q(s,a)|^p = |\expec_{( s',a') \sim \cP(\cdot \,|\, s,a) }[Q(s',a')]|^p \le\expec_{( s',a') \sim \cP(\cdot \,|\, s,a) }[|Q(s',a')|^p]) = \cP |Q|^p(s,a)$ by Jensen's inequality, $\rho^\top |\cP Q|^p \le \rho^\top \cP |Q|^p \le C_{\mu,\rho}\mu^\top |Q|^p$.
\end{proof}

Now, we are ready to prove Proposition \ref{prop::1}.
\begin{proof}[Proof of Proposition \ref{prop::1}] 
    By Lemma \ref{lem::1}, 
\begin{align*}
   &g^{\pi_\star}-g^{\pi_K} \\&\le \cP^{\pi_K}_{*}\bigg(\sum^K_{l=0}\Pi^K_{i=l+1}\lambda_{i} (\lambda_{l+1}-\lambda_{l}) \Pi^K_{i=l+1}\cP^{\pi_{i}}\left(\sum^l_{m=0} \Pi^l_{i=m+1}\lambda_{i} 
 (1-\lambda_{m}) (\cP^{\pi_\star})^{l+1-m}-I\right)(Q^0-Q^{\pi_\star})\\&+\sum^K_{l=1}\Pi^K_{i=l}\lambda_{i}\left(\sum^K_{m=l} (\lambda_{m+1}-\lambda_{m}) \Pi^K_{i=m+1}\cP^{\pi_{i}}  
 (\cP^{\pi_\star})^{m+1-l}+  \Pi^K_{i=l+1}\cP^{\pi_{i}} (\lambda_l\cP^{\pi_{l}}-I)\right) \epsilon_l\bigg)
  \\& \le \cP^{\pi_K}_{*}\bigg(\sum^K_{l=0}\Pi^K_{i=l+1}\lambda_{i} (\lambda_{l+1}-\lambda_{l}) \Pi^K_{i=l+1}\cP^{\pi_{i}}\left(\sum^l_{m=0} \Pi^l_{i=m+1}\lambda_{i} 
 (1-\lambda_{m}) (\cP^{\pi_\star})^{l+1-m}+I\right)|Q^0-Q^{\pi_\star}|\\&+\sum^K_{l=1}\Pi^K_{i=l}\lambda_{i}\left(\sum^K_{m=l} (\lambda_{m+1}-\lambda_{m}) \Pi^K_{i=m+1}\cP^{\pi_{i}}  
 (\cP^{\pi_\star})^{m+1-l}+  \Pi^K_{i=l+1}\cP^{\pi_{i}} (\lambda_l\cP^{\pi_{l}}+I)\right) |\epsilon_l|\bigg).
\end{align*}     
Let $\cP^Q_l=\cP^{\pi_K}_{*} \Pi^K_{i=l+1}\cP^{\pi_{i}}\left(\sum^l_{m=0} \Pi^l_{i=m+1}\lambda_{i} 
 (1-\lambda_{m}) (\cP^{\pi_\star})^{l+1-m}+I\right)/2$ and $\cP^\epsilon_l=\cP^{\pi_K}_{*} \sum^K_{m=l} (\lambda_{m+1}-\lambda_{m}) \Pi^K_{i=m+1}\cP^{\pi_{i}}  
 (\cP^{\pi_\star})^{m+1-l}+  \Pi^K_{i=l+1}\cP^{\pi_{i}} (\lambda_l\cP^{\pi_{l}}+I)/2$. Then $\cP^Q_l$ and $\cP^\epsilon_l$ satisfying $ \rho^\top \cP^Q_l \le C_{\mu,\rho}\mu$ and $ \rho^\top \cP^\epsilon_l \le C_{\mu,\rho}\mu$ for all $ 0 \le l \le K$ by Lemma \ref{lem_conv} and \ref{lem_opt}. Thus, we have 
\begin{align*}
  \| g^{\pi_\star}-g^{\pi_K}\|_{p, \rho} &\le 2\sum^K_{l=0}\Pi^K_{i=l+1}\lambda_{i} (\lambda_{l+1}-\lambda_{l})\|\cP_l |Q^0-Q^{\pi_\star}|\|_{p,\rho}+2\sum^K_{l=1}\Pi^K_{i=l}\lambda_i\|\cP^\epsilon_l |\epsilon_l| \|_{p,\rho}
   \\&\le 2C^{1/p}_{\mu}\sum^K_{l=0}\Pi^K_{i=l+1}\lambda_{i} (\lambda_{l+1}-\lambda_{l})\|Q^0-Q^{\pi_\star}\|_{p,\mu}+2C^{1/p}_{\mu}\sum^K_{l=1}\Pi^K_{i=l}\lambda_i\|\epsilon_l \|_{p,\mu},
\end{align*}
where last inequality comes from Lemma \ref{lem_cov}. By plugging $\lambda_k=\frac{k}{k+2}$, we conclude.
Note that since $C_\mu \le C_{\mu,\rho} $ for any distribution $\rho$, then choosing $\rho$  to be a Dirac distribution at each state proves the case of Assumption \ref{converage_tran} which implies first inequality of Proposition \ref{prop::1}. 
% $\infn{\cdot}$-norm result.
\end{proof}
\section{Omitted proofs in Section \ref{sec::sam_comp}}\label{app:sec_4}

\subsection{Proof of Lemma \ref{appx_bel_iid}}

\begin{proof}[Proof of Lemma \ref{appx_bel_iid}]
Let $\cF \subset \{f : \cS \times \cA \rightarrow [- f_{max},  f_{max}] \,|\, f \in B(S\times A)\}$ and $\cG \subset \{f : \cS \times \cA \rightarrow [- g_{max},  g_{max}] \,|\, f \in B(S\times A)\}$.
   Let $f_1,\dots, f_N $ cover the $\cF$ and  $g_1,\dots, g_{N'} $ cover the $\cG$ where $N=\cN(\epsilon/M; \cF, \infn{\cdot})$, $N'=\cN(\epsilon/M; \cG, \infn{\cdot})$, $M=108(R+2f_{max})$. $\cF \times \cG  = \cup S_{i,j}$ where $S_{i,j} = \{ (f,g) \,:\, \infn{f-f_i} \le \epsilon,  \infn{g-g_j} \le\epsilon\}$. Without loss of generality, suppose $g_{max} \le f_{max}$.
   
   First note that $ \mathbb{E}_{s_i'\sim P(\cdot \,|\, s_i,a_i)}[r(s_i,a_i)+\max_a g(s'_i, a)] = Tg(s_i,a_i)$, $|r_i+ \max_a g(s,a)| \le R+ f_{max}$, $|Tg(s,a)| \le  R+ f_{max}$. 
   
   For arbitrary $f\in \cF,g \in \cG$, define
$X^{f,g}_i =(f(s_i, a_i)-r(s_i,a_i)-\max_a g(s'_i, a))^2- (Tg(s_i, a_i)-r(s_i,a_i)-\max_a g(s'_i, a))^2$. Then, $\mathbb{E}_{s_i,a_i\sim\mu,s'_i\sim P(\cdot \,|\, s_i,a_i)}[X^{f,g}_i] = \|Tg-f\|^2_{\mu,2}$ and 
$\mathbb{E}[(X^{f,g}_i)^2] \le 9 ( R+2f_{max} )^2\|Tg-f\|^2_{\mu,2}$ since $X^{f,g}_i =(f(s_i, a_i)-Tg(s_i, a_i))(f(s_i, a_i)+Tg(s_i, a_i)-2r(s_i,a_i)-2\max_a g(s'_i, a))$, and $|X^{f,g}_i| \le 3 ( R+2f_{max} )^2$. 
% and $Var(X^{f,g}_i) \le \|Tg-f\|^2_{\mu,2}10 ( R+2f_{max} )^2$.

%  By Bernstein's inequality, for fixed $f,g$,
% \[  P\left(\|Tg-f\|^2_{\mu,2}-\sum^n_{i=1}X^{f,g}_i/n  \ge \epsilon\right)\le exp(xx)\]

% by previous lemma, we have
% \[  P\left( \sup_{f,g}\|Tg-f\|^2_{\mu,2}-\sum^n_{i=1}X^{f,g}_i/n  \ge \epsilon\right)\le  N^2(\epsilon/(4M); \cF, \infn{\cdot})  exp(xx)\]

By Bernstein inequality and union bound, with $1-\delta$ probability, for all $\{f_i, g_j\}_{i=1,\dots, N, j=1,\dots, N'}$,
\begin{align*}
    \|Tg_j-f_i\|^2_{\mu,2}-\sum^n_{i=1}X^{f_i,g_j}_i/n&\le \sqrt{\frac{18( R+2f_{max} )^2\|Tg_j-f_i\|^2_{\mu,2} \ln (\cN_{\cF,\cG}/\delta)}{n}}\\&\quad +\frac{2( R+2f_{max} )^2 \ln (\cN_{\cF,\cG}/\delta)}{n}
\end{align*}  
where $\cN_{\cF,\cG} =\cN(\epsilon/M; \cG, \infn{\cdot})\cN(\epsilon/M; \cF, \infn{\cdot})$. Through $2\sqrt{ab} \le a+b$, we have
\[ \forall f_i \in \cF, \forall g_i \in \cG,\quad \|Tg_j-f_i\|^2_{\mu,2}-2\sum^n_{i=1}X^{f_i,g_j}_i/n\le \frac{22( R+2f_{max} )^2 \ln (\cN_{\cF,\cG}/\delta)}{n}\]
Now, for covering number argument, we use following Lemma.
\begin{lemma}
For $f\in \cF$, $g\in \cG$, $c>0$,    $\|Tg-f\|^2_{\mu,2}-c\sum^n_{i=1}X^{f,g}_i/n$ is $(2+8c)(2f_{max}+R)$- Lipchitz.
\end{lemma}
\begin{proof}
Since $\|Tg_1-f_1\|^2_{\mu,2}-\|Tg_2-f_2\|^2_{\mu,2} \le \expec |(Tg_1-Tg_2+f_2-f_1)(Tg_1+Tg_2-f_2-f_1)| \le (\infn{g_1-g_2}+\infn{f_1-f_2})2(R+2f_{max})$, $\|Tg-f\|^2_{\mu,2}$ is $2(R+2f_{max})$- Lipchitz. Also, since  $|\sum^n_{i=1}X^{f_1,g_1}_i/n-\sum^n_{i=1}X^{f_2,g_2}_i/n|= \frac{1}{n}\sum^n_{i=1}|(\max  g_2-\max g_1+f_1-f_2)(f_2+f_1-\max g_1- \max  g_2-2r)-(Tg_1-Tg_2+\max g_2-\max g_1)(Tg_1+Tg_2+\max g_2+\max g_1-2r)| \le (\infn{g_1-g_2}+\infn{f_1-f_2})2(R+2f_{max})+8\infn{g_1-g_2}(f_{max}+R) \le (\infn{g_1-g_2}+\infn{f_1-f_2})8(2f_{max}+R)$,  $\sum^n_{i=1}X^{f_1,g_1}_i/n$ $8(2f_{max}+R)$-Lipchitz. By adding two Lipchitz functions, we obtain desired result.
\end{proof}

By Lipchitzness of $\|Tg_j-f_i\|^2_{\mu,2}-2\sum^n_{i=1}X^{f_i,g_j}_i/n$ and definition of covering number, if $f,g \in S_{i,j}$
\begin{align*}
\|Tg-f\|^2_{\mu,2}-2\sum^n_{i=1}X^{f,g}_i/n-(\|Tg_j-f_i\|^2_{\mu,2}-2\sum^n_{i=1}X^{f_i,g_j}_i/n) \le \epsilon.    
\end{align*}
% and
% \[  P(\|Tg-f\|^2_{\mu,2}-2\sum^n_{i=1}X^{f,g}_i/n  \ge \epsilon) \le \sum_{i,j}P(\sup_{S_{i,j}s\in f_i,g_j}\| Tg_j-f_i\|^2_{\mu,2}-2\sum^n_{i=1}X^{f_i,g_j}_i/n \ge \epsilon/2).\]
This implies that with $1-\delta$ probability,  
\begin{align}\label{eq1}
    \forall f \in \cF, \forall g \in \cG \quad \|Tg-f\|^2_{\mu,2}\le \epsilon+\frac{22( R+2f_{max} )^2 \ln (\cN_{\cF,\cG}/\delta)}{n}+2\sum^n_{i=1}X^{f,g}_i/n.
\end{align} 
        By other side of Bernstein's inequality and covering number, for all $\{f_i, g_j\}_{i=1,\dots, N, j=1,\dots, N'}$, we have
    \begin{align*}
        \sum^n_{i=1}X^{f_i,g_j}_i/n -  \|Tg_j-f_i\|^2_{\mu,2} &\le \sqrt{\frac{18( R+2f_{max})^2\|Tg_j-f_i\|^2_{\mu,2} \ln (\cN_{\cF,\cG}/\delta)}{n}}\\&\quad +\frac{2( R+2f_{max} )^2 \ln (\cN_{\cF,\cG}/\delta)}{n}.
    \end{align*}

    If $\sum^n_{i=1}X^{f_i,g_j}_i/n  \ge \frac{4( R+2f_{max} )^2 \ln (\cN_{\cF,\cG}/\delta)}{n}$, with $1-\delta$ probability, for all $\{f_i, g_j\}_{i=1,\dots, N, j=1,\dots, N'}$, we have

        \[\sum^n_{i=1}X^{f_i,g_j}_i/n -  \|Tg_j-f_i\|^2_{\mu,2}\le \sqrt{4.5\sum^n_{i=1}X^{f_i,g_j}_i/n \|Tg_j-f_i\|^2_{\mu,2} }+\frac{2( R+2f_{max})^2 \ln (\cN_{\cF,\cG}/\delta)}{n}\]
        and by $2\sqrt{ab} \le a+b$, this implies
        \[\sum^n_{i=1}X^{f_i,g_j}_i/n- 6.5\|Tg_j-f_i\|^2_{\mu,2}\le  \frac{4( R+2f_{max})^2 \ln (\cN_{\cF,\cG}/\delta)}{n}.\]
Even if $\sum^n_{i=1}X^{f_i,g_j}_i/n  \le \frac{4( R+2f_{max} )^2 \ln (\cN_{\cF,\cG}/\delta)}{n}$, previous inequality still holds. Since $\sum^n_{i=1}X^{f_i,g_j}_i/n- 6.5\|Tg_j-f_i\|^2_{\mu,2}$ is $54( R+2f_{max})$-Lipshitz, with similar argument, we have
        \begin{align}\label{eq2}
            \forall f \in \cF,g \in \cG,  \quad \sum^n_{i=1}X^{f,g}_i/n- 6.5\|Tg-f\|^2_{\mu,2}\le \epsilon+ \frac{4( R+2f_{max})^2 \ln (\cN_{\cF,\cG}/\delta)}{n}.
        \end{align}
    Let $\tilde{T}g = \argmin_{f \in \cF} \|f-Tg\|_{2,\mu}$ and $f=\tilde{T}g$ in inequality (\ref{eq2}). Then, by definition of Inherent Bellman error, 
\[\forall g \in \cG, \quad \sum^n_{i=1}X^{\tilde{T}g,g}_i/n\le   \epsilon+6.5\epsilon_B+\frac{4( R+2f_{max})^2 \ln (\cN_{\cF,\cG}/\delta)}{n}.\]
%     By Bernstein's inequality and covering number, for all $f_i, g_j$, we have
%     \[\sum^n_{i=1}X^{\tilde{T}g_j}_i/n -  \|Tg_j-\tilde{T}g_j\|^2_{\mu,2} \le \sqrt{\frac{18( R+2f_{max})^2\|Tg_j-\tilde{T}g_j\|^2_{\mu,2} \ln (N/\delta)}{n}}+\frac{4( R+2f_{max} )^2 \ln (N/\delta)}{3n}\]

%     If $\sum^n_{i=1}X^{\tilde{T}g_j}_i/n  \ge \frac{( R+2f_{max} )^2 \ln (1/\delta)}{n}$, we have

%         \[\sum^n_{i=1}X^{\tilde{T}g_j}_i/n -  \|Tg_j-\tilde{T}g_j\|^2_{\mu,2}\le \sqrt{18\sum^n_{i=1}X^{\tilde{T}g_j}_i/n \|Tg_j-\tilde{T}g_j\|^2_{\mu,2} }+\frac{4( R+2f_{max})^2 \ln (N/\delta)}{3n}\]
%         and by $2\sqrt{ab} \le a+b$, this implies
%         \[\sum^n_{i=1}X^{\tilde{T}g_j}_i/n\le   20\|Tg_j-\tilde{T}g_j\|^2_{\mu,2}+\frac{8( R+2f_{max})^2 \ln (N/\delta)}{3n}\]

% By assumption, 
% \[\sum^n_{i=1}X^{\tilde{T}g_j}_i/n\le   20\epsilon_B+\frac{8( R+2f_{max})^2 \ln (N/\delta)}{3n}.\]
Also, let $f= \hat{T}g$ in inequality inequality (\ref{eq1}). Then,  by definition of $\hat{T}g$, we have $\sum^n_{i=1}X^{\hat{T}g,g}_i \le \sum^n_{i=1}X^{\tilde{T}g,g}_i$. Combining with previous inequality, with $1-2\delta$ probability, 
\[ \forall g \in \cG, \quad \|Tg-\hat{T}g\|^2_{\mu,2}\le 3\epsilon+13\epsilon_B+\frac{30( R+2f_{max} )^2 \ln (\cN_{\cF,\cG}/\delta)}{n}.\]
Finally, let $\cG=\cF_{k}, \cF =\cF_{k+1},$ and $g=f_k$, and  by manipulating $\delta$, we get desired result.
% Finally, by considering $\delta/(2K)$, we get desired result.
\end{proof}

% Now, we are ready to prove Theorem \ref{Sam_com_iid}.
\subsection{Proof of Theorem \ref{Sam_com_iid}}
\begin{proof}[Proof of Theorem \ref{Sam_com_iid}]

By combining Lemma \ref{appx_bel_iid} and Proposition \ref{prop::1}, we directly obtain following results. Under assumptions
     stated in Theorem \ref{appx_bel_iid}, we have
       \begin{align*}
       \| g^{\pi_\star}-g^{\pi_K}\|_{\infty}
       &\le C^{1/2}_\mu\frac{8\|Q^{\pi_\star}\|_{ 2, \mu}}{K+2}
       \\&\!\!\!\!\!\!\!\!\!\!\!\!+C^{1/2}_\mu\frac{2K}{3}\left(\sqrt{3\epsilon'}+\sqrt{\frac{60(K+1)^2R^2  \ln(2KN^2_{\epsilon'}/\delta)}{n}}+\underset{k=0,\dots,K-1}{\max} \sqrt{13\epsilon_B(\cF_k,\cF_{k+1})} \right),
    \end{align*}
           \begin{align*}
       \| g^{\pi_\star}-g^{\pi_K}\|_{2,\rho}
       &\le C^{1/2}_{\mu,\rho}\frac{8\|Q^{\pi_\star}\|_{ 2, \mu}}{K+2}
       \\&\!\!\!\!\!\!\!\!\!\!\!\!+C^{1/2}_{\mu,\rho}\frac{2K}{3}\left(\sqrt{3\epsilon'}+\sqrt{\frac{60(K+1)^2R^2  \ln(2KN^2_{\epsilon'}/\delta)}{n}}+\underset{k=0,\dots,K-1}{\max} \sqrt{13\epsilon_B(\cF_k,\cF_{k+1})} \right),
    \end{align*}
    where 
    \[
N_\epsilon'=\max_{k=1,...,K}N_{k,\epsilon'},\qquad
N_{k,\epsilon}=\cN\big(\tfrac{\epsilon'}{108(2k+1)R}; \cF_k, \infn{\cdot}\big),\quad\text{for }k=1,\dots,K.
      \]
   Given $\epsilon>0$, for the first inequality, let  $ K=\lceil 18C^{1/2}_\mu\|Q^{\pi_\star}\|_{2,\mu}/\epsilon \rceil, \epsilon'= \frac{4\epsilon^2}{27K^2C_{\mu}}, n= \frac{36K^2C_{\mu}}{\epsilon^2}60 R^2(K+1)^2 \ln (2K\cN^2_{\epsilon'}/\delta)$. Then, by direct calculation, we derive that 
     \[\infn{g^{\pi_\star}-g^{\pi_K}} \le \epsilon +3KC^{1/2}_\mu\underset{k=0,\dots,K-1}{\max} \sqrt{\epsilon_B(\cF_k,\cF_{k+1})}\]
     with sample complexity
     % \[n=O\left(\frac{10^10 \|Q^{\pi_\star}\|_{2,\mu}^4C^3_{\mu}R^2}{\epsilon^6}\ln (36\cN^2_{\epsilon}C^{1/2}_{\mu}/(\delta\epsilon))\right)\]
     \[n=\mathcal{O}\left(\frac{ \|Q^{\pi_\star}\|_{2,\mu}^4C^3_{\mu}R^2}{\epsilon^6}\ln (\cN^2_{\epsilon}C^{1/2}_{\mu}/(\delta\epsilon))\right)\]
where
           \[
N_\epsilon=\max_{k=1,...,K}N_{k,\epsilon},\qquad
N_{k,\epsilon}=\cN\big(\tfrac{\epsilon^4}{10^6kC^2_\mu \|Q^{\pi_\star}\|^2_{2,\mu}R}; \cF_k, \infn{\cdot}\big),\quad\text{for }k=1,\dots,K.
      \]
     Similarly,  given $\epsilon>0$, for second inequality, let $K=\lceil 18C^{1/2}_{\mu,\rho}\|Q^{\pi_\star}\|_{2,\mu}/\epsilon \rceil, \epsilon'= \frac{4\epsilon^2}{27K^2C_{\mu,\rho}},  n= \frac{36K^2C_{\mu,\rho}}{\epsilon^2}60 R^2(K+1)^2 \ln (2K^2\cN_{\epsilon'}/\delta),$ and 
     \[\|g^{\pi_\star}-g^{\pi_K}\|_{2,\rho} \le \epsilon +3KC^{1/2}_{\mu,\rho} \underset{k=0,\dots,K-1}{\max} \sqrt{\epsilon_B(\cF_k,\cF_{k+1})}\]
     with sample complexity
         \[n=\mathcal{O}\left(\frac{ \|Q^{\pi_\star}\|_{2,\mu}^4C^3_{\mu,\rho}R^2}{\epsilon^6}\ln (\cN^2_{\epsilon}C^{1/2}_{\mu}/(\delta\epsilon))\right)\]
         where 
          \[
N_\epsilon=\max_{k=1,...,K}N_{k,\epsilon},\qquad
N_{k,\epsilon}=\cN(\tfrac{\epsilon^4}{10^6kC^2_{\mu,\rho}\|Q^{\pi_\star}\|^2_{2,\mu}R}; \cF_k, \infn{\cdot}),\quad\text{for }k=1,\dots,K.
      \]
\end{proof}
\subsection{Proof of Lemma \ref{appx_bel_traj}}

We first introduce empirical covering number.  
\begin{definition}[empirical covering number]
  For a given function class $\cF$ of real valued functions and set $x^{1:n} = (x_1,\dots, x_n)$, denote the covering number of $\cF$ equipped with the empirical $l_1$ pseudo metric 
  $l_{x^{1:n}}(f,g) =\frac{1}{n}\sum^n_{i=1}|f(x_i)-g(x_i)|$ by
  $\cN(\epsilon, \cF, x^{1:n})$.
\end{definition}
Although the empirical convering number depends on number of samples, but it can be bounded by pseudo dimension which depends on only function space and $\epsilon$ as following fact shows.

\begin{fact}[\cite{haussler1995sphere}\label{pseudo_dimension}, Corollary 3]
    For any $x^{1:n}=(x_1,\dots, x_n)$, any function class $\cF$ of real-valued functions taking values in $[0,M]$ with pseudo-dimension $V_{\cF} <\infty$, and any $\epsilon>0$,
    \[\cN(\epsilon, \cF, l_{x^{1:N}}) \le e(V_{\cF} +1)\left(\frac{2eM}{\epsilon}\right)^{V_{\cF}}.\]
\end{fact}

Define $L(g,f) = \expec_{s_i,a_i \sim \mu}[Var_{s'_i\sim P( \,|\, s_i,a_i)}(r(s_i,a_i)+\max f (s'_i, a))] +\|g-T f \|^2_{2,\mu}$ where $Var$ denotes variance with respect to $s_i'$, and $\hat{L}(g,f)=\frac{1}{n}\sum^n_{i=1}(g(s_i, a_i)-r(s_i,a_i)-\max_a f(s'_i, a))^2$. Then, $\expec[\hat{L}(g,f)]=L(g,f)$ and following lemma holds.
\begin{lemma}\label{lem_objec}
    $\|\hat{T}f-Tf\|^2_{2,\mu}- \inf_{g \in \cG} \|g-Tf\|^2_{2,\mu} \le 2 \sup_{g \in \cG}|L(g,f)-\hat{L}(g,f)|$.
\end{lemma}
\begin{proof}[Proof of Lemma \ref{lem_objec}]
% Since $\expec [ (\hat{T}f(s_i, a_i)-r(s_i,a_i)-\max_a f(s'_i, a))^2] = \expec[Var(r(s_i,a_i)+\max f (s_i, a))] +\expec [(\hat{T}f(s_i, a_i)-T f (s_i, a_i))^2]$, we have 
$\|\hat{T}f-Tf\|^2_{2,\mu}- \inf_{g \in \cG} \|g-Tf\|^2_{2,\mu} = L(\hat{T}f,f) - \inf_{g \in \cF} L(g,f) = L(\hat{T}f,f)-\hat{L}(\hat{T}f,f)+\hat{L}(\hat{T}f,f) - \inf_{g \in \cG} L(g,f) \le 2 \sup_{f \in \cF}|L(g,f)-\hat{L}(g,f)|$ by definition of $\hat{T}f$. 
\end{proof}
For $\{\hat{T}f_k,f_k\}^{K-1}_{k=0}$ of Anc-F-QI, previous lemma implies that
\begin{align*}
    \|\hat{T}f_k-Tf_k\|^2_{2,\mu}- \inf_{g \in \cG} \|g-Tf_k\|^2_{2,\mu} &\le \sup_{f\in \cF}(\|\hat{T}f-Tf\|^2_{2,\mu}-\inf_{g \in \cG} \|g-Tf\|^2_{2,\mu}) 
    \\&\le  2 \sup_{g\in \cG ,f \in \cF}|L(g,f)-\hat{L}(g,f)|.
\end{align*}

Define the function $l_{f,g} : \cS \times \cA \times [-R, R] \times \cS \rightarrow \real$ as $l_{f,g}(s_i,a_i,r_i,s_{i+1})= (f(s_i,a_i)-r_i-\max_{a}g(s_{i+1},a))^2$ and the function space $\cL_{\cF,\cG} = \{l_{f,g} \,|\, f\in \cF, g \in \cG\}$ and $\cG_{max}= \{\max_a g(s,a) \,|\, g \in \cG\}$. The pseudo dimension of $\cG_{max}$ could be bounded by following Lemma.
\begin{lemma}\label{lem_max_vc}
Define $\cG_{max} =\{\max_{a\in \cA}g(\cdot, a) \,:\, g\in \cG \}$.    $V_{\cG_{max}} \le 2|\cA|V_{\cG} \log(3|\cA|)$ .
\end{lemma}
\begin{proof}[Proof of Lemma \ref{lem_max_vc}]
By the definition of pseudo dimension, we have $V_\cG \ge V_{\cG^i}$ where $\cG^i=\{g(x,a_i) \,|\, g \in \cG)\}$. Since     $\max_{a\in \cA}g(\cdot, a)  \le 0 \iff \forall i\,\, g(\cdot, a_i)  \le 0 $, the claim follows from Lemma 3.2.3 of \cite{blumer1989learnability}. 
\end{proof}

% Here, we only denote dependency of $n$ in sample complexity due to complexity of expression. Please refer to Appendix for exact dependency of other parameters. 
% /Compared to Lemma \ref{appx_bel_iid}, sample error of single-trajectory decreases slower rate, and 
 % Due to this lemma we can bound $\|\epsilon_k\|_{2,\mu}$ with high probabity in our Theorem \ref{prop::1}, which is equal to $\|Tf_k-\hat{T}f_k\|_{\mu,2}$. Thus, incorporationg previous Theorem, we directly obtain following sample complexity reusult. 
 Now, we are ready to prove Lemma \ref{appx_bel_traj}.
\begin{proof}[Proof of Lemma \ref{appx_bel_traj}]
Let $\cF \subset \{f : \cS \times \cA \rightarrow [- f_{max},  f_{max}] \,|\, f \in B(S\times A)\}$ and $\cG \subset \{g : \cS \times \cA \rightarrow [- g_{max},  g_{max}] \,|\, g \in B(S\times A)\}$. Without loss of generality, $g_{max} \le f_{max}$.

By similar argument in proof of Proposition 4 of \cite{carrasco2002mixing},  $\{s_i,a_i, r_i\}$ is
also $\beta$-mixing with the coefficient $\{\beta_i\}$  and this implies $\{s_i,a_i, r_i, s_{i+1}\}$ is also stationary $\beta$-mixing with coefficient $\{\beta_{i-1}\}$. By direct calculation, $|\hat{L}(f,g)| \le (2f_{max}+R)^2$. Now, we apply Fact \ref{hoeff_mix} with $l(f,g)$ and $Z_i=(s_i,a_i,r_i,s_{i+1})$. Then, we get
 \[P\left(\sup_{f\in \cF,g\in \cG} \left|\hat{L}(f,g)-L(f,g)\right|>\epsilon\right) \le 16 \expec [\cN ( \epsilon/8, \cL_{\cF,\cG}, (Z'_t)_{t \in H})]e^{-\frac{m_N \epsilon^2}{128(2f_{max}+R)^4}}+2m_N \beta_{k_N}. \]
Since  
\begin{align*}
    &\hat{L}(f_1,g_1) - \hat{L}(f_2,g_2) 
    \\&=\frac{1}{n}\left|\sum^n_{i=1}(f_1(s_i,a_i)-r(s_i,a_i)-\max_{a\in \cA}g_1(s_{i+1},a))^2-\sum^n_{i=1}(f_2(s_i,a_i)-r(s_i,a_i)-\max_{a\in \cA}g_2(s_{i+1},a))^2\right|\\&\le2\frac{2f_{max}+R}{n}\sum^n_{i=1}(|f_1(s_i,a_i)-f_2(s_i,a_i)|+|\max_{a\in \cA} g_1(s_{i+1},a)-\max_{a\in \cA} g_2(s_{i+1},a)|),
\end{align*}
this implies that 
\[\cN (4(2f_{max}+R)\epsilon, \cL_{\cF,\cG}, (z^{1:n}) \le \cN(\epsilon, \cF, s^{2:n+1}) \cN(\epsilon, \cG_{max}, (s,a)^{1:n}) \]
where $z_i=(s_i,a_i, r_i, s_{i+1})$ by definition of empirical covering number.
Finally, by Fact \ref{pseudo_dimension}, we get
\begin{align*}
   & \cN (\epsilon/8, \cL_{\cF,\cG}, (Z'_t)_{t \in H}) \\&\le e(V_{\cF} +1)\left(\frac{128(2f_{max}+R)e}{\epsilon}\right)^{V_{\cF}} e(V_{\cF_{max}} +1)\left(\frac{128(2f_{max}+R)e}{\epsilon}\right)^{V_{\cG_{max}}} \\& =C \left(\frac{1}{\epsilon}\right)^{V_{\cF}+V_{\cG_{max}}}
\end{align*}
where $C=e^2(V_{\cF} +1)(V_{\cG_{max}} +1)(128(2f_{max}+R)e)^{V_{\cF}+V_{\cG_{max}}}.$

For calculation, we use following prior result. 
\begin{fact}[\cite{antos2008learning}, Lemma 14]
Let $\beta_m \le \bar{\beta}e^{(-bm^\kappa)}, N\ge 1, k_N = \lceil (C_2 N\epsilon^2/b)^{\frac{1}{1+\kappa}} \rceil , m_N = N/(2k_N), 0<\delta \le 1, V\ge 2$ and $C_1,C_2,\bar{\beta}, b, \kappa>0$. Define $\epsilon$ and $C_0$ as
\[\epsilon = \sqrt{\frac{C_0(\max \{C_0/b, 1\})^{1/\kappa}}{C_2N}}\]
with $C_0= V/2 \log N+\log (e/\delta)+\log(\max (C_1C_2^{V/2},\bar{\beta},1))$
\[    C_1 \left(\frac{1}{\epsilon}\right)^{V}e^{-4C_2m_N \epsilon^2}+2m_N \beta_{k_N} \le \delta. \]
\end{fact}

Then, by this fact and previous arguments, for $\epsilon=\sqrt{\frac{c_0(\max \{c_0/b, 1\})^{1/\kappa}}{c_2n}}$,
\[P\left(\sup_{f \in \cF,g\in \cG} \left|\hat{L}(f,g)-L(f,g)\right|\le \epsilon\right) \ge 1-\delta \]
where $c_0= (V_{\cF}+V_{\cG_{max}})/2 \log n+\log (e/\delta)+\log(\max (c_1c_2^{(V_{\cF}+V_{\cG_{max}})/2},\bar{\beta},1), c_1=16e^2(V_{\cF} +1)(V_{\cG_{max}} +1)(128(2f_{max}+R)e 2)^{V_{\cF}+V_{\cG_{max}}}, c_2=\frac{1}{512(2f_{max}+R)^4}, V_{\cG_{max}}= 2|\cA|V_{\cG} \log(3|\cA|)$. Let $\cG=\cF_{k}, \cF =\cF_{k+1}$ and $g=f_k$. By Lemma \ref{lem_objec}, this implies that with $1-\delta$ probability,
\[\|Tf_k-\hat{T}f_k\|^2_{\mu,2} \le \epsilon_B+ \sqrt{\frac{c_0(\max \{c_0/b, 1\})^{1/\kappa}}{4c_2n}}.\]
Finally, by manipulating $\delta$, we get desired result.
\end{proof}
\subsection{Proof of Theorem \ref{Sam_com_traj}}
% Now, we are ready to prove  Theorem \ref{Sam_com_traj}.
\begin{proof}[Proof of Theorem \ref{Sam_com_traj}]

By combining Lemma \ref{appx_bel_traj} and Proposition \ref{prop::1}, we directly obtain following results. Under assumptions
     stated in Theorem \ref{appx_bel_traj}, we have
       \begin{align*}
       \| g^{\pi_\star}-g^{\pi_K}\|_{\infty}
       &\le C^{1/2}_\mu\frac{8\|Q^{\pi_\star}\|_{ 2, \mu}}{K+2}
       \\&+C^{1/2}_\mu\frac{2K}{3}\left(\left(\frac{c_{0,K}(\max \{c_{0,K}/b, 1\})^{1/\kappa}}{c_{2,K}n}\right)^{1/4}+\underset{k=0,\dots,K-1}{\max} \sqrt{\epsilon_B(\cF_k,\cF_{k+1})} \right),
    \end{align*}
           \begin{align*}
       \| g^{\pi_\star}-g^{\pi_K}\|_{2,\rho}
       &\le C^{1/2}_{\mu,\rho}\frac{8\|Q^{\pi_\star}\|_{ 2, \mu}}{K+2}
       \\&+C^{1/2}_{\mu,\rho}\frac{2K}{3}\left(\left(\frac{c_{0,K}(\max \{c_{0,K}/b, 1\})^{1/\kappa}}{c_{2,K}n}\right)^{1/4}+\underset{k=0,\dots,K-1}{\max} \sqrt{\epsilon_B(\cF_k,\cF_{k+1})} \right),
    \end{align*}
where $c_{0,K} = \max_{k=0,\dots,K-1} c_{0,k}, c_{0,k}= (V_{\cF_{k+1}}+V_{(\cF_{k})_{max}})/2 \log n+\log (e/(K\delta))+\log(\max (c_{1,k},\bar{\beta},1)), c_{1,k}=16e^2(V_{\cF_{k+1}} +1)(V_{(\cF_{k})_{max}} +1) (24e)^{V_{\cF_{k+1}}+V_{(\cF_{k})_{max}}},c_{2,K}=\frac{1}{512(2K+1)^4R^4}, V_{(\cF_{k})_{max}}= 2|\cA|, V_{\cF_k} \log(3|\cA|)$.

   Given $\epsilon>0$, for the first inequality, let  $ K=\lceil 9C^{1/2}_\mu\|Q^{\pi_\star}\|_{2,\mu}/\epsilon\rceil$. Then, by direct calculation, we derive that 
     \[\infn{g^{\pi_\star}-g^{\pi_K}} \le \epsilon +KC^{1/2}_\mu\underset{k=0,\dots,K-1}{\max} \sqrt{\epsilon_B(\cF_k,\cF_{k+1})}\]
     with sample complexity
     % \[n=O\left(\frac{10^10 \|Q^{\pi_\star}\|_{2,\mu}^4C^3_{\mu}R^2}{\epsilon^6}\ln (36\cN^2_{\epsilon}C^{1/2}_{\mu}/(\delta\epsilon))\right)\]
     \[n=\tilde{\mathcal{O}}\left(\frac{b^{-1/\kappa}(c'_{0,K})^{\frac{1+\kappa}{\kappa}}R^4\|Q^{\pi_\star}\|_{2,\mu}^8C^{6}_\mu}{\epsilon^{12}}\right)\]
     where $c'_{0,K} = \max_{k=0,\dots,K-1} c'_{0,k}, c'_{0,k}= \log (1/\delta)+\log(\max (c_{1,k},\bar{\beta})), c_{1,k}=16e^2(V_{\cF_{k+1}} +1)(V_{(\cF_{k})_{max}} +1) (24e)^{V_{\cF_{k+1}}+V_{(\cF_{k})_{max}}}, V_{(\cF_{k})_{max}}= 2|\cA|, V_{\cF_k} \log(3|\cA|)$, and $\tilde{\mathcal{O}}$ ignores all logarithmic factors.
     
     % \[n=\tilde{\mathcal{O}}\left(\frac{b^{1/\kappa}R^2K^6}{(V_{\cF_{k+1}}+V_{(\cF_{k})_{max}})^{(\kappa+1)/(2\kappa)}\epsilon^4}\log (Ke/\delta))^{(\kappa+1)/(2\kappa)}\right).\]
   Similarly, given $\epsilon>0$, for the second inequality, let  $ K=\lceil 9C^{1/2}_{\mu,\rho}\|Q^{\pi_\star}\|_{2,\mu}/\epsilon\rceil$. Then, by direct calculation, we derive that 
     \[\infn{g^{\pi_\star}-g^{\pi_K}} \le \epsilon +KC^{1/2}_{\mu,\rho}\underset{k=0,\dots,K-1}{\max} \sqrt{\epsilon_B(\cF_k,\cF_{k+1})}\]
     with sample complexity
     % \[n=O\left(\frac{10^10 \|Q^{\pi_\star}\|_{2,\mu}^4C^3_{\mu}R^2}{\epsilon^6}\ln (36\cN^2_{\epsilon}C^{1/2}_{\mu}/(\delta\epsilon))\right)\]
     \[n=\tilde{\mathcal{O}}\left(\frac{b^{-1/\kappa}(c'_{0,K})^{\frac{1+\kappa}{\kappa}}R^4\|Q^{\pi_\star}\|_{2,\mu}^8C^{6}_{\mu,\rho}}{\epsilon^{12}}\right)\]
     where $c'_{0,K} = \max_{k=0,\dots,K-1} c'_{0,k}, c'_{0,k}= \log (1/\delta)+\log(\max (c_{1,k},\bar{\beta})), c_{1,k}=16e^2(V_{\cF_{k+1}} +1)(V_{(\cF_{k})_{max}} +1) (24e)^{V_{\cF_{k+1}}+V_{(\cF_{k})_{max}}}, V_{(\cF_{k})_{max}}= 2|\cA|, V_{\cF_k} \log(3|\cA|)$, and $\tilde{\mathcal{O}}$ ignores all logarithmic factors.
\end{proof}

\section{Omitted proofs in Section \ref{sec::rel_anc}}

\subsection{Proof of Theorem \ref{Sam_com_iid_rel}}
We first prove following key lemma.
\begin{lemma}\label{appx_bel_iid_rel}
Assume Assumptions \ref{assum_bell_opt}, \ref{assum_ex_arg}, \ref{assum_star_fun}, \ref{iid_data}, \ref{assum_nor_fun}, and \ref{assum_ran_fun} (Bellman optimality equation, existence of argmin, star-shaped function space, normalized function space, range of function space, IID dataset). Let $\mu$ be the distribution generating the dataset. Let $\epsilon>0$ and $\delta>0$. 
% Assumption \ref{cvx_cpt_fc},\ref{bellman_err}, \ref{iid_data},     
With probability $1-\delta$, $\{f_k, \hat{T}f_k\}^{K-1}_{k=0}$ of R-Anc-F-QI satisfies 
\[\|Tf_k-\hat{T}f_k\|^2_{\mu,2} \le \frac{30(R+4\infn{Q^{\pi_\star}})^2 \ln(2KN^2_{\epsilon}/\delta)}{n}+3\epsilon+13\epsilon_B(\cF,\cF),\]
where 
\[
N_{\epsilon}=\cN(\tfrac{\epsilon}{108(R+4\infn{Q^{\pi_\star}})}; \cF, \infn{\cdot}).
\]
% \[\|Tf_k-\hat{T}f_k\|^2_{\mu,2} \le \frac{60(k+2)^2R^2  \ln(2KN_{k,\epsilon}N_{k+1,\epsilon}/\delta)}{n}+3\epsilon+13\epsilon_B,\]
% where 
% \[
% N_{k,\epsilon}=\cN(\tfrac{\epsilon}{108(2k+1)R)}; \cF_k, \infn{\cdot}).
% \]
\end{lemma}
\begin{proof}    
The proof basically follows from the proof of Lemma \ref{appx_bel_iid}.
\end{proof}
% \begin{lemma}
% XXXweakly or BellmanXXX Under Assumptions \ref{assum_ex_arg}, \ref{assum_star_fun}, \ref{assum_inc_fun},\ref{iid_data},\ref{assum_nor_fun} (existenc of argmin, star-shaped function space, increasing function range, normalized function space, IID dataset), 
% % Assumption \ref{cvx_cpt_fc},\ref{bellman_err}, \ref{iid_data},     
% with probability $1-\delta$, for given $\epsilon, \delta>0$ and $\{f_k\}^{K-1}_{k=0}$ of R-Anc-F-QI, we have
% \[\|Tf_k-\hat{T}f_k\|^2_{\mu,2} \le \frac{30(R+4\infn{Q^{\pi_\star}})^2 \ln(2KN^2_{\epsilon}/\delta)}{n}+3\epsilon+13\epsilon_B,\]
% where
% \[
% N_{\epsilon}=\cN(\frac{\epsilon}{108(R+4\infn{Q^{\pi_\star}})}; \cF, \infn{\cdot}).
% \]
% \end{lemma}

Now, we prove Theorem \ref{Sam_com_iid_rel}.
\begin{proof}[Proof of Theorem \ref{Sam_com_iid_rel}]
Consider Apporximate Relative Anchored Value Iteration 
\begin{align}
Q_r^{k} = (1-\lambda_k) Q_r^0+\lambda_k (TQ_r^{k-1}+\epsilon_k- c_k\mathbf{1})
\tag{Apx-R-Anc-QI}\label{eq:Rel-Apx-Anc-QI}
\end{align}
for $c_k \in \real$. Also, consider corresponding Approximate Anchored Value Iteration with same $\epsilon_k$ and starting point $Q_r^0$
\begin{align}
Q^{k} = (1-\lambda_k) Q_r^0+\lambda_k (TQ^{k-1}+\epsilon_k).
\tag{Apx-Anc-QI}
\end{align}
Since $Q^k-Q_r^{k} = d_k \mathbf{1}$ for some $d_k \in \real$, $\max_{a} Q^k(s,a)=\max_{a} Q_r^k(s,a)$ for all $s \in \cS$ by the defintion of Bellman operator and this implies induced policies are same.
% $TQ^{k}-Q^{k} = TQ_r^{k}-Q_r^{k}$ and $\argmax Q_r^{k} = \argmax Q^{k}$. 
Thus, Proposition~\ref{prop::1} also holds for \ref{eq:Rel-Apx-Anc-QI}. 

By combining Lemma \ref{appx_bel_iid_rel} and Proposition \ref{prop::1}, we directly obtain following results. Under assumptions
     stated in Theorem \ref{Sam_com_iid_rel},
      \begin{align*}
       \| g^{\pi_\star}-g^{\pi_K}\|_{\infty}
       &\le C^{1/2}_\mu\frac{8\|Q^{\pi_\star}\|_{ 2, \mu}}{K+2}
       \\&+C^{1/2}_\mu\frac{2K}{3}\left(\sqrt{3\epsilon'}+\sqrt{\frac{30(R+4\infn{Q^{\pi_\star}})^2 \ln(2KN^2_{\epsilon'}/\delta)}{n}}+\sqrt{13\epsilon_B(\cF,\cF)} \right).
    \end{align*}
  \begin{align*}
       \| g^{\pi_\star}-g^{\pi_K}\|_{2,\rho}
       &\le C^{1/2}_{\mu,\rho}\frac{8\|Q^{\pi_\star}\|_{ 2, \mu}}{K+2}
       \\&+C^{1/2}_{\mu,\rho}\frac{2K}{3}\left(\sqrt{3\epsilon'}+\sqrt{\frac{30(R+4\infn{Q^{\pi_\star}})^2\ln(2KN^2_{\epsilon'}/\delta)}{n}}+\sqrt{13\epsilon_B(\cF,\cF)} \right),
    \end{align*}
    where 
    \[
N_{\epsilon'}=\cN\big(\tfrac{\epsilon'}{108(R+4\infn{Q^{\pi_\star}})}; \cF, \infn{\cdot}\big).\]
   Given $\epsilon>0$, for the first inequality, let  $ K=\lceil 18C^{1/2}_\mu\|Q^{\pi_\star}\|_{2,\mu}/\epsilon \rceil, \epsilon'= \frac{4\epsilon^2}{27K^2C_{\mu}}, n= \frac{36K^2C_{\mu}}{\epsilon^2}30(R+4\infn{Q^{\pi_\star}})^2\ln(2KN^2_{\epsilon'}/\delta)$. Then, by direct calculation, we derive that 
     \[\infn{g^{\pi_\star}-g^{\pi_K}} \le \epsilon +3KC^{1/2}_\mu\sqrt{\epsilon_B(\cF,\cF)}\]
     with sample complexity
     % \[n=O\left(\frac{10^10 \|Q^{\pi_\star}\|_{2,\mu}^4C^3_{\mu}R^2}{\epsilon^6}\ln (36\cN^2_{\epsilon}C^{1/2}_{\mu}/(\delta\epsilon))\right)\]
     \[n=\mathcal{O}\left(\frac{ (R+\infn{Q^{\pi_\star}})^2\infn{Q^{\pi_\star}}^2C^3_{\mu}}{\epsilon^4}\ln (\cN^2_{\epsilon}C^{1/2}_{\mu}/(\delta\epsilon))\right)\]
where
    \[
N_{\epsilon}=\cN\big(\tfrac{\epsilon^4}{10^6C^2_{\mu}(R+\infn{Q^{\pi_\star}})\infn{Q^{\pi_\star}}^2}; \cF, \infn{\cdot}\big).\]
     Similarly,  given $\epsilon>0$, for second inequality, let $K=\lceil 18C^{1/2}_{\mu,\rho}\|Q^{\pi_\star}\|_{2,\mu}/\epsilon\rceil, \epsilon'= \frac{4\epsilon^2}{27K^2C_{\mu,\rho}},  n= \frac{36K^2C_{\mu,\rho}}{\epsilon^2}30(R+4\infn{Q^{\pi_\star}})^2 \ln (2K^2\cN_{\epsilon'}/\delta),$ and 
     \[\|g^{\pi_\star}-g^{\pi_K}\|_{2,\rho} \le \epsilon +3KC^{1/2}_{\mu,\rho} \sqrt{\epsilon_B(\cF,\cF)}\]
     with sample complexity
     % \[n=O\left(\frac{10^10 \|Q^{\pi_\star}\|_{2,\mu}^4C^3_{\mu}R^2}{\epsilon^4}\ln (36\cN^2_{\epsilon}C^{1/2}_{\mu}/(\delta\epsilon))\right)\]
     \[n=\mathcal{O}\left(\frac{ (R+\infn{Q^{\pi_\star}})^2\infn{Q^{\pi_\star}}^2C^3_{\mu,\rho}}{\epsilon^4}\ln (\cN^2_{\epsilon}C^{1/2}_{\mu,\rho}/(\delta\epsilon))\right)\]
where
    \[
N_{\epsilon}=\cN\big(\tfrac{\epsilon^4}{10^6C^2_{\mu,\rho}(R+\infn{Q^{\pi_\star}})\infn{Q^{\pi_\star}}^2}; \cF, \infn{\cdot}\big).\]
\end{proof}
\subsection{Proof of Theorem \ref{Sam_com_traj_rel}}
We first prove following key Lemma.

\begin{lemma}\label{appx_bel_traj_rel}
Assume Assumptions \ref{assum_bell_opt}, \ref{assum_ex_arg},  \ref{assum_star_fun}, \ref{traj_data}, \ref{traj_beta_data}, \ref{assum_nor_fun}, and \ref{assum_ran_fun} (Bellman optimality equation, existence of argmin, star-shaped function space,  normalized function space, range of function space, single-trajectory dataset, $\beta$-mixing single-trajectory). Let $\mu$ be the distribution generating the dataset defined as $\mu(s,a)= \nu(s)\pi_b(a \,|\, s) $. Let $\epsilon>0$ and $\delta>0$. 
% Assumption \ref{cvx_cpt_fc},\ref{bellman_err}, \ref{iid_data},     
With probability $1-\delta$, $\{f_k, \hat{T}f_k\}^{K-1}_{k=0}$ of R-Anc-F-QI satisfies
\[\|Tf_k-\hat{T}f_k\|^2_{\mu,2} \le \epsilon_B(\cF,\cF)+ \sqrt{\frac{c_{0}(\max \{c_{0}/b, 1\})^{1/\kappa}}{c_{2}n}}\]
where $c_{0}= (V_{\cF}+V_{\cF_{max}})\log n/2 +\log (e/(K\delta))+\log(\max (c_{1},\bar{\beta})), c_{1}=16e^2(V_{\cF} +1)(V_{\cF_{max}} +1) (24e)^{V_{\cF}+V_{\cF_{max}}}, c_{2}=\frac{1}{512(R+4\infn{Q^{\pi_\star}})^4}, V_{\cF_{max}}= 2|\cA|V_{\cF} \log(3|\cA|)$.
%     With probability $1-\delta$, we have
% \[\|Tf_k-\hat{T}f_k\|^2_{\mu,2} \le 3\epsilon^2_B+ \sqrt{\frac{42( R+2f_{max})^2 \ln (V_{\cF_k})/\delta)}{n}}\]
\end{lemma}
\begin{proof}    
The proof basically follows from the proof of Lemma \ref{appx_bel_traj}.
\end{proof}

Now, we prove  Theorem \ref{Sam_com_traj_rel}.
\begin{proof}[Proof of Theorem \ref{Sam_com_traj_rel}]

By combining Lemma \ref{appx_bel_traj_rel} and Proposition \ref{prop::1}, we directly obtain following results. Under assumptions
     stated in Theorem \ref{appx_bel_traj_rel}, we have
       \begin{align*}
       \| g^{\pi_\star}-g^{\pi_K}\|_{\infty}
       &\le C^{1/2}_\mu\frac{8\|Q^{\pi_\star}\|_{ 2, \mu}}{K+2}
       \\&+C^{1/2}_\mu\frac{2K}{3}\left(\left(\frac{c_{0}(\max \{c_{0}/b, 1\})^{1/\kappa}}{c_{2}n}\right)^{1/4}+ \sqrt{\epsilon_B(\cF,\cF)} \right),
    \end{align*}
           \begin{align*}
       \| g^{\pi_\star}-g^{\pi_K}\|_{2,\rho}
       &\le C^{1/2}_{\mu,\rho}\frac{8\|Q^{\pi_\star}\|_{ 2, \mu}}{K+2}
       \\&+C^{1/2}_{\mu,\rho}\frac{2K}{3}\left(\left(\frac{c_{0}(\max \{c_{0}/b, 1\})^{1/\kappa}}{c_{2}n}\right)^{1/4}+ \sqrt{\epsilon_B(\cF,\cF)} \right),
    \end{align*}
where $c_{0}= (V_{\cF}+V_{\cF_{max}})/2 \log n+\log (e/(K\delta))+\log(\max (c_{1},\bar{\beta},1)), c_{1}=16e^2(V_{\cF} +1)(V_{\cF_{max}} +1) (24e)^{V_{\cF}+V_{\cF_{max}}}, c_{2}=\frac{1}{512(R+4\infn{Q^{\pi_\star}})^4}, V_{\cF_{max}}= 2|\cA|V_{\cF} \log(3|\cA|)$.

   Given $\epsilon>0$, for the first inequality, let  $ K=\lceil 9C^{1/2}_\mu\|Q^{\pi_\star}\|_{2,\mu}/\epsilon\rceil$. Then, by direct calculation, we derive that 
     \[\infn{g^{\pi_\star}-g^{\pi_K}} \le \epsilon +KC^{1/2}_\mu \sqrt{\epsilon_B(\cF,\cF)}\]
     with sample complexity
     % \[n=O\left(\frac{10^10 \|Q^{\pi_\star}\|_{2,\mu}^4C^3_{\mu}R^2}{\epsilon^6}\ln (36\cN^2_{\epsilon}C^{1/2}_{\mu}/(\delta\epsilon))\right)\]
     \[n=\tilde{\mathcal{O}}\left(\frac{b^{-1/\kappa}(c'_{0})^{\frac{1+\kappa}{\kappa}}(R+\infn{Q^{\pi_\star}})^4\|Q^{\pi_\star}\|_{\infty}^4C^{4}_\mu}{\epsilon^{8}}\right)\]
where $c'_{0}= \log (1/\delta)+\log(\max (c_{1},\bar{\beta})), c_{1}=16e^2(V_{\cF} +1)(V_{\cF_{max}} +1) (24e)^{V_{\cF}+V_{\cF_{max}}}, V_{\cF_{max}}= 2|\cA|V_{\cF} \log(3|\cA|)$, and $\tilde{\mathcal{O}}$ ignores all logarithmic factors.
     
     % \[n=\tilde{\mathcal{O}}\left(\frac{b^{1/\kappa}R^2K^6}{(V_{\cF_{k+1}}+V_{(\cF_{k})_{max}})^{(\kappa+1)/(2\kappa)}\epsilon^4}\log (Ke/\delta))^{(\kappa+1)/(2\kappa)}\right).\]
   Similarly, given $\epsilon>0$, for the second inequality, let  $ K=\lceil 9C^{1/2}_{\mu,\rho}\|Q^{\pi_\star}\|_{2,\mu}/\epsilon\rceil$. Then, by direct calculation, we derive that 
     \[\infn{g^{\pi_\star}-g^{\pi_K}} \le \epsilon +KC^{1/2}_{\mu,\rho} \sqrt{\epsilon_B(\cF,\cF)}\]
     with sample complexity
     % \[n=O\left(\frac{10^10 \|Q^{\pi_\star}\|_{2,\mu}^4C^3_{\mu}R^2}{\epsilon^6}\ln (36\cN^2_{\epsilon}C^{1/2}_{\mu}/(\delta\epsilon))\right)\]
     \[n=\tilde{\mathcal{O}}\left(\frac{b^{-1/\kappa}(c'_{0})^{\frac{1+\kappa}{\kappa}}(R+\infn{Q^{\pi_\star}})^4\|Q^{\pi_\star}\|_{\infty}^4C^{4}_{\mu,\rho}}{\epsilon^{8}}\right)\]
where $c'_{0}= \log (1/\delta)+\log(\max (c_{1},\bar{\beta})), c_{1}=16e^2(V_{\cF} +1)(V_{\cF_{max}} +1) (24e)^{V_{\cF}+V_{\cF_{max}}}, V_{\cF_{max}}= 2|\cA|V_{\cF} \log(3|\cA|)$, and $\tilde{\mathcal{O}}$ ignores all logarithmic factors.
\end{proof}

\end{document}